\documentclass{article}
\usepackage[margin=1in,footskip=0.25in]{geometry} 

\usepackage[utf8]{inputenc} 
\usepackage[T1]{fontenc}    
\usepackage{hyperref}       
\usepackage{url}            
\usepackage{booktabs}       
\usepackage{amsfonts}       
\usepackage{nicefrac}       
\usepackage{microtype}      
\usepackage{amsmath}
\usepackage{amsfonts}
\usepackage{amssymb}
\usepackage{graphicx}
\usepackage{verbatim}
\usepackage{authblk}
\usepackage{graphicx}
\usepackage[square,sort,comma,numbers]{natbib}
\usepackage{doi}
\usepackage{times}
\usepackage{amsthm}
\usepackage{mathrsfs}
\usepackage{authblk}

\theoremstyle{definition} 
\newtheorem{theorem}{Theorem}

\newtheorem{lemma}{Lemma}
\newtheorem{definition}{Definition}
\newtheorem{problem}{Problem}
\usepackage{amsmath,amsfonts,amssymb,natbib, graphicx, url}
\usepackage[algo2e]{algorithm2e} 
\usepackage{mathtools}
\usepackage{commath}
\newtheorem{claim}{Claim}
\usepackage[utf8]{inputenc}
\usepackage[scaled=0.92]{helvet}
\usepackage{color, colortbl}
\definecolor{lightGray}{gray}{0.9}

\usepackage{multirow}
\usepackage{tikz}
\usepackage{subcaption}
\usepackage{subfiles}
\usepackage{pgfplots}
\usepackage{algorithm}
\usepackage{algorithmic}
\usepackage{caption}
\pgfplotsset{scaled y ticks=false}
\pgfplotsset{every axis/.append style={
		scale=0.95,
		label style={font=\small},
		tick label style={font=\small},
	},
	compat=1.17}

\newcommand{\relu}{\text{ReLU}}

\newcommand{\E}{\mathbb{E}}
\newcommand{\R}{\mathbb{R}}

\newcommand{\X}{\mathbf{X}}

\newcommand{\Q}{\mathbf{Q}}
\renewcommand{\S}{\mathbf{S}}
\newcommand{\y}{\mathbf{y}}
\renewcommand{\v}{\mathbf{v}}
\renewcommand{\u}{\mathbf{u}}
\newcommand{\x}{\mathbf{x}}

\newcommand{\w}{\mathbf{w}}
\newcommand{\wh}{\hat{\w}}
\newcommand{\eps}{\varepsilon}
\DeclareMathOperator{\Var}{Var}
\newcommand{\z}{\mathbf{z}}

\DeclareMathOperator*{\rank}{rank}

\newcommand{\g}{\mathbf{g}}
\DeclareMathOperator*{\argmin}{arg\,min}

\newtheorem*{rep@theorem}{\rep@title}
\newcommand{\newreptheorem}[2]{%
	\newenvironment{rep#1}[1]{%
		\def\rep@title{#2 \ref{##1}}%
		\begin{rep@theorem}}%
		{\end{rep@theorem}}}
\makeatother
\newreptheorem{theorem}{Theorem}
\usepackage{multirow}
\usepackage{array}
\usepackage{url}

\title {Active Learning for Single Neuron Models with Lipschitz Non-Linearities}

\author{\small{Aarshvi Gajjar}}
\author{Chinmay Hegde}
\author{Christopher Musco}

\affil{New York University}
\date{}
\usepackage[sc,osf]{mathpazo}

  \usepackage{nth}
  \usepackage{intcalc}

  \newcommand{\cAAAI}[1]{AAAI\ Conference\ on\ Artificial (AAAI)}

\begin{document}
\maketitle
\begin{abstract}%
We consider the problem of active learning for single neuron models, also sometimes called ``ridge functions'', in the agnostic setting (under adversarial label noise). Such models have been shown to be broadly effective in modeling physical phenomena, and for constructing surrogate data-driven models for partial differential equations. 
 
 Surprisingly, we show that for a single neuron model with any Lipschitz non-linearity (such as the ReLU, sigmoid, absolute value, low-degree polynomial, among others), strong provable approximation guarantees can be obtained using a well-known active learning strategy for fitting \emph{linear functions} in the agnostic setting.
Namely, we can collect samples via statistical \emph{leverage score sampling}, which has been shown to be near-optimal in other active learning scenarios. We support our theoretical results with empirical simulations s  howing that our proposed active learning strategy based on leverage score sampling outperforms (ordinary) uniform sampling when fitting single neuron models. 
\end{abstract}


\section{Introduction}
This paper considers active learning methods for functions of the form $g(\x) = f(\langle \w,\x\rangle)$, where $\w$ is a weight vector and $f$ is a non-linearity. For a given distribution $\mathcal{D}$ on $\R^d \times \R$, a random vector $(\x,y)$ sampled from $\mathcal{D}$, and scalar function $f:\R\rightarrow \R$, our goal is to find $\w$ which minimizes the expected squared error:
\begin{align*}
	\E_{\x,y\sim \mathcal{D}} \left(f(\langle \w,\x\rangle) - y\right)^2.
\end{align*}
Functions of the form $f(\langle \w,\x\rangle)$ find applications in a variety of settings under various names: they are called ``single neuron'' or ``single index'' models, ``ridge functions'', and ``plane waves'' in different communities \citep{Pinkus:1997,Pinkus:2015,YehudaiOhad:2020,RaoGantiBalzano:2017,Candes:2003}. Single neuron models are studied in machine learning theory as tractable examples of simple neural networks \citep{DiakonikolasGoelKarmalkar:2020,GoelKanadeKlivans:2017}. Moreover, these models are known to be adept at modeling a variety of physical phenomena \citep{ConstantineRosarioIaccarino:2016} and for that reason are effective e.g., in building surrogate models for efficiently solving parametric partial differential equations (PDEs) and for approximating quantity of interest (QoI) surfaces for uncertainty quantification, model-driven design, and data assimilation \citep{OLeary-RoseberryVillaChen:2022,ConstantineEftekhariHokanson:2017,CohenDaubechiesDeVore:2012,Le-MaitreKnio:2010,LassilaRozza:2010,BinevCohenDahmen:2017}.

In these applications, single neuron models are used to fit complex functions over $\R^d$ based on queries from those functions. Often, the cost of obtaining a query $(\x,y)$ from the target function dominates the computational cost of fitting the model: each training point collected requires numerically solving the PDE under consideration with parameters given by $\x$ \citep{AdcockBrugiapagliaDexter:2022,CohenDeVore:2015}.
At the same time, we often have the {freedom} in exactly \emph{how} the query is obtained; we are not restricted to simply sampling from $\mathcal{D}$, but rather can specify a target location $\x$ and sample $y \sim \mathcal{D}\mid \x$ (or compute $y$ deterministically, since in most applications it is a deterministic function of $\x$).
Given these considerations, the focus of our work is on developing efficient \emph{active learning} and {experimental design} methods\footnote{We use ``experimental design'' to refer to methods that collect samples in a non-adaptive way. In other words, a set of points $\x_1, \ldots, \x_s$ are specified upfront and the corresponding $\y$ values are observed all at once. In contrast, in standard active learning methods, the choice of $\x_j$ can depend on the response values of all prior points $\x_1, \ldots, \x_{j-1}$.} for fitting single neuron models using as few carefully chosen $(\x,y)$ observations as possible.

We study this active learning problem in the challenging \emph{agnostic learning} or adversarial noise setting. Again, this is motivated by applications of single neuron models in computational science. Typically, while it can be approximated by a single neuron model, the QoI or surrogate function under consideration is not itself of the form $f(\langle \w,\x\rangle)$. For this reason, the agnostic setting has become the standard in work on PDE models involving other common function families, like structured polynomials \citep{ChkifaDexterTran:2018,CohenDeVore:2015,AdcockCardenasDexter:2022,HamptonDoostan:2015b}. In the agnostic setting, for a constant $C$ (or more stringently, $C = 1+\eps$), our goal is always to return with high probability some $\tilde{\w}$ such that:
\begin{align*}
\E_{\x,y\sim \mathcal{D}} \left(f(\langle \tilde{\w},\x\rangle) - y\right)^2 \leq C \cdot \min_{\w} \E_{\x,y\sim \mathcal{D}} \left(f(\langle {\w},\x\rangle) - y\right)^2.
\end{align*}
\subsection{Our Contributions}
For ease of exposition, we consider the case when $\mathcal{D}$ is a uniform distribution over $n$ points in $\R^d$. This is without loss of generality, since any continuous distribution can be approximated by the uniform distribution over a sufficient large, but finite, subset of points in $\R^d$.\footnote{For other function families (e.g. polynomials, or sparse Fourier functions) there has been recent work on active learning algorithms based on leverage score sampling that skip the discrete approximation step by developing algorithms directly tailored to common continuous distributions, like uniform or Gaussian \cite{ErdelyiMuscoMusco:2020}. We believe our techniques should be directly extendable to give comparable results for single neuron models.}.
 In this case, we have the following problem statement. 

\begin{problem}[Single Neuron Regression]\label{prob1}
	Given a matrix $\X \in \R^{n \times d}$ and query access to a target vector $\y \in \R^n$, for a given function $f:\R\rightarrow\R$, find $\w \in \R^d$ to minimize $\norm{f(\X\w) - \y}_2^2$ using as few queries from $\y$ as possible.
\end{problem}

When $f$ is an identity function, Problem \ref{prob1} reduces to active least squares regression, which has received a lot of recent attention in computer science and machine learning.  In the agnostic setting, state-of-the-art results can be obtained via ``leverage score'' sampling, also known as  ``coherence motivated'' or ``effective resistance'' sampling  \citep{AvronKapralovMusco:2019,CohenMigliorati:2017,RauhutWard:2012,HamptonDoostan:2015}. The idea behind leverage scores sampling methods is to collect samples from $\y$ randomly but non-uniformly, using an importance sampling distribution based on the rows of $\X$. More ``unique'' rows are selected with higher probability. Formally, rows are selected with probability proportional to their statistical leverage scores: 
\begin{definition}[Statistical Leverage Score]\label{def:lev_scores} The leverage score, $\tau_i(\X)$ of the $i^\text{th}$ row, $\x_i$ of a matrix, $\X \in \R^{n \times d}$ is equal to:\vspace{-1em}
	\begin{align*}
		\tau_i(\X) = \x_i^T (\X^T \X)^{-1} \x_i = \max_{\w \in \R^d}\frac{[\X\w]_i^2}{\norm{\X\w}_2^2} \, .
	\end{align*}
		Here, $[\X\w]_i$ denotes the $i^\text{th}$ entry of the vector $\X\w$. 
\end{definition}
We always have that $0 \leq \tau_i(\X) \leq 1$, and a well-known property of the statistical leverage scores is that $\sum_{i=1}^n \tau_i(\X) = \rank(\X) \leq d$.
The leverage score of a row is large (close to $1$) if that row has large inner product with some vector in $\R^d$, as compared to that vector's inner product with all other rows in the matrix $\X$. This means that the particular row enjoys significance in forming the row space of $\X$. 
For linear regression, it can be shown that when $\X$ has $d$ columns, leverage score sampling yields a sample complexity of $O(d\log d + d/\eps)$ to find $\hat{\w}$ satisfying $\|\X\hat{\w}- \y\|_2^2 \leq (1+\eps)\min_{\w}\|\X\w- \y\|_2^2$; moreover, this is optimal up to the $\log$ factor \citep{ChenPrice:2019a}. 

Our main contribution is to establish that, surprisingly, when combined with a novel regularization strategy, leverage score sampling \emph{also} yields theoretical guarantees for the more general case (Problem \ref{prob1}) for a broad class of non-linearities $f$. In fact, we only require that $f$ is $L$-Lipschitz for constant $L$, a  property that holds for most non-linearities used in practice (ReLU, sigmoid, absolute value, low-degree polynomials, etc.). 
We prove the following main result, which shows that $\tilde{O}(d^2/\eps^4)$ samples, collected via leverage score sampling, suffice for provably learning a single neuron model with Lipschitz non-linearity. 


\begin{theorem}[Main Result]\label{thm:main_result}
	Let $\X\in \R^{n\times d}$ be a data matrix and $\y\in\R^n$ be a target vector. Let $f$ be an $L$-Lipschitz non-linearity with $f(0) = 0$ and let $OPT = \min_{\w} \|f(\X{\w}) - \y\|_2^2 $. 
	There is an algorithm (Algorithm \ref{alg:active-lean}) that, based on the leverage scores of $\X$, observes $m = O\left(\frac{d^2 \log(1/\eps)}{\eps^4}\right)$ random entries from $\y$ and returns with probability $> 9/10$ a vector $\hat{\w}$ satisfying:
	\begin{align*}
		\|f(\X\hat{\w}) - \y\|_2^2 \leq C \cdot \left(OPT + \eps L^2 \|\X\w^*\|_2^2\right).
	\end{align*}
\end{theorem}

The assumption $f(0) = 0$ in Theorem \ref{thm:main_result} is without loss of generality. If $f(0)$ is non-zero, we can simply solve a transformed problem with $\y' = \y - f(0)$ and $f'(x) = f(x) - f(0)$. The theorem mirrors previous results in the linear setting, and in contrast to prior work on agnostically  learning single neuron models, does not require any assumptions on $\X$ \citep{DiakonikolasKontonisTzamos:2022,TyagiCevher:2012}. In addition to multiplicative error $C$, the theorem has an additive error term of $C\eps L^2 \|\X\w^*\|_2^2$. An additive error term is necessary;  as we will show in Section \ref{sec:limitations}, it is provably impossible to achieve purely relative error with a number of samples polynomial in $d$. Similar additive error terms arise in related work on leverage score sampling for problems like logistic regression \citep{MunteanuSchwiegelshohnSohler:2018,MaiRaoMusco:2021}. On the other hand, we believe the $d^2$ dependence in our bound is not necessary, and should be improvable to linear in $d$. The dependence on $\eps$ is also likely improvable.

In Section \ref{sec:exp} we support our main theoretical result with an empirical evaluation on both synthetic data and several test problems that require approximating PDE quantity of interest surfaces. In all settings, leverage score sampling outperforms ordinary uniform sampling, often improving on the error obtained using a fixed number of samples by an order of magnitude or more.

\subsection{Related Work}

Single neuron models have been widely studied in a number of communities, including machine learning, computational science, and approximation theory. These models can be generalized to the ``multi-index'' case, where $g(\x) = f_1(\langle \x,\w_1) + \ldots + f_q(\langle \x,\w_1)$ \citep{KlusowskiBarron:2018,TyagiCevher:2012,Candes:2003} or to the case when $f$ is not known in advance (but  might be from a parameterized function family, such as low-degree polynomials) \citep{CohenDaubechiesDeVore:2012,HokansonConstantine:2018,ConstantineEftekhariHokanson:2017}. While we do not address these generalizations in this paper, we hope that our work can provide a foundation for further work in this direction.

Beyond sample complexity, there has also been a lot of interest in understanding the computational complexity of fitting single neuron models, including in the agnostic setting \citep{YehudaiOhad:2020}. There are both known hardness results for general data distributions~ \citep{DiakonikolasKaneZarifis:2020,GoelKarmalkarKlivans:2019,DiakonikolasKaneManurangsi:2022}, as well as positive results on efficient algorithms under additional assumptions~\citep{GoelKanadeKlivans:2017,DiakonikolasKontonisTzamos:2022,DiakonikolasGoelKarmalkar:2020}. 
Our work differs from this setting in two ways: 1) we focus on sample complexity; 2) we make no assumptions on the data distribution $\mathcal{D}$ (i.e., no assumptions on the data matrix $\X$); 3) we allow for active sampling strategies. Obtaining results under i.i.d.\ sampling, even for well behaved distributions, inherently requires additional assumptions, like $\w^* = \min_{\w}\|f(\X\w)-\y\|_2$ being bounded in norm or $\X$ having bounded condition number.

In computational science, there has been significant work on active learning and experimental design for fitting other classes of functions adept at modeling high dimensional physical phenomena~\citep{ChkifaDexterTran:2018,CohenDeVore:2015,AdcockCardenasDexter:2022,HamptonDoostan:2015b}. Most such results focus on minimizing squared prediction error. For situations involving model mismatch, the agnostic (adversarial noise) setting is more appropriate than assuming i.i.d.\, zero mean noise, which is more typical in classical statistical results on experimental design \citep{Pukelsheim:2006}. To obtain results in the agnostic setting for \emph{linear} function families, much of the prior work in computational science uses ``coherence motivated'' sampling techniques \citep{CohenMigliorati:2017,RauhutWard:2012,HamptonDoostan:2015}. Such methods are equivalent to leverage score sampling \citep{AvronKapralovMusco:2019}.

Leverage score sampling and related methods have found widespread applications in the design of efficient algorithms, including in the construction of graph sparsifiers, coresets, randomized low-rank approximations, and in solving over-constrained linear regression problems \citep{SpielmanSrivastava:2011,FeldmanLangberg:2011,DrineasMahoneyMuthukrishnan:2008,CohenMuscoMusco:2017,MuscoMusco:2017,DasguptaDrineasHarb:2008,DrineasMahoneyMuthukrishnan:2006,CohenLeeMusco:2015}. 

Recently, there has been renewed attention on using leverage score sampling for solving the active linear regression problem for various loss functions \citep{ErdelyiMuscoMusco:2020,ChenDerezinski:2021,MeyerMuscoMusco:2023, MuscoMuscoWoodruff:2022}. While all of these results are heavily tailored to linear models, a small body of works addresses the problem of active learning for non-linear problems. This includes problems of the form $\|f(\X\w - \y)\|_2^2$, where $f$ is Lipschitz \citep{MunteanuSchwiegelshohnSohler:2018,MaiRaoMusco:2021}. While not equivalent to our problem, this formulation captures important tasks like logistic regression. Our work can be viewed as broadening the range of non-linear problems for which leverage score sampling yields natural worst-case guarantees. 

Finally, we mention a few papers that directly address the active learning problem for functions of form $f(\X\w)$. \cite{CohenDaubechiesDeVore:2012} studies adaptive query algorithms, but in a different setting than our paper. Specifically, they address the easier (noiseless) realizable setting, where samples are obtained from a function of the form $f(\X\w^*)$ for a ground truth $\w^*$. They also make stronger smoothness assumptions on $f$, although their algorithm can handle the case when $f$ is not known in advance. Follow-up work also address the multi-index problem in the same setting \citep{FornasierSchnassVybiral:2012}. Also motivated by applications in efficient PDE surrogate modeling, \cite{TyagiCevher:2012} study the multi-index problem, but again in the realizable setting. Their techniques can handle mean centered i.i.d.\ noise, but not adversarial noise or model mismatch. 
\section{Preliminaries}

\textbf{Notation.} Throughout the paper, we use bold lower-case letters for vectors and bold upper-case letters for matrices. We let $\mathbf{e}_i$ denotes the $i^\text{th}$ standard basis vector (all zeros, but with a $1$ in position $i$). The dimension of $\mathbf{e}_i$ will be clear from context.  For a vector $\x\in \R^n$ with entries $x_1, \ldots, x_n$, $\|\x\|_2 = (\sum_{i=1}^n x_i^2)^{1/2}$ denotes the Euclidean norm. $\mathcal{B}^d(r)$ denotes a ball of radius $r$ centered at $0$, i.e. $\mathcal{B}^d(r) = \{\x \in \R^d: \norm{\x}\le r \}$. For a scalar function $f:\R\rightarrow \R$ and vector $\y$, we use $f(\y)$ to denote the vector obtained by applying $f$ to each entry of $\y$. For a fixed matrix $\X$, unobserved target vector $\y$, and non-linearity $f$, we denote $\|f(\X\w^*) - \y\|_2^2$ by $OPT$ where $\w^* \in \argmin_{\w} \|f(\X\w) - \y\|_2$. 

\noindent \textbf{Importance Sampling.} As discussed, our approach is based on importance sampling according to statistical leverage scores: we fix a set of probabilities $p_1, \ldots, p_n$, use them to sample $m$ rows from the regression problem $\norm{f(\X\w) - \y}_2^2$, and solve a reweighted least squares problem to find a near optimal choice for $\w$. Formally, this can be implemented by defining a sampling matrices of the following form:\\
\begin{definition}[Importance Sampling Matrix]\label{def:imp_samp_matrix}
	Let $\{p_1, \ldots, p_n\} \in [0,1]^n$ be a given set of probabilities (so that $\sum_i p_i = 1$). A matrix
	$\S$ is an $m \times n$ \textit{importance sampling matrix} if each of its rows is chosen to equal $\frac{1}{\sqrt{m\cdot p_i}}\cdot \mathbf{e}_i$ with probability proportional to $p_i$. 
\end{definition}
To compute an approximate to $\min \norm{f(\X\w) - \y}_2^2$ we will solve an optimization problem involving the sub-sampled objective $\norm{\S f(\X\w) - \S\y}_2^2$. It is easily verified that for \emph{any} choice of $p_1, \ldots, p_n$ and \emph{any} vector $\z$, 
$\E\left[ \norm{\S\z}_2^2\right] = \norm{\z}_2^2.$
We will use this fact repeatedly.\\

\noindent \textbf{Properties of Leverage Scores.} Our importance sampling mechanism is based on sampling by the leverage scores $\tau_1(\X), \ldots, \tau_n(\X)$ of the design matrix $\X\in \R^{d\times n}$. For \emph{any} full-rank $d\times d$ matrix $\mathbf{R}$, we have that $\tau_i(\X\mathbf{R}) = \tau_i(\X)$. This is clear from Definition \ref{def:lev_scores} and implies that $\tau_i(\X)$ only depends on the column span of $\X$. In our proofs, this property will allow us to easily reduce to the setting where $\X$ is assumed to be a matrix with orthonormal columns. 

We will also use the following well-known fact about using leverage score sampling to construct a ``subspace embedding'' for a matrix $\X$. 

\begin{lemma}[Subspace Embedding (see e.g. Theorem 17 in \citet{Woodruff:2014}]
	\label{lem:subspace}
	Given $\X\in \R^{n\times d}$ with leverage scores $\tau_1(\X), \ldots, \tau_n(\X)$, let $p_i  = \tau_i(\X)/\sum_i\tau_i(\X)$. Let $\S \in \R^{m\times n}$ be a sampling matrix constructed as in Definition \ref{def:imp_samp_matrix} using the probabilities $p_1, \ldots, p_n$. For any $0< \gamma < 1$, as long as $m \geq c\cdot d\log(d/\delta)/\gamma^2$ for some fixed constant $c$, then with probability $1-\delta$ we have that for all $\w \in \R^d$,
	\begin{align*}
		(1-\gamma) \|\X\w\|_2^2 \leq \|\S\X\w\|_2^2 \leq (1+\gamma) \|\X\w\|_2^2.
	\end{align*}
\end{lemma}

Lemma \ref{lem:subspace} establishes that, with high probability, leverage score sampling preserves the norm of any vector $\X\w$ in the column span of $\X$. This guarantee can be proven using an argument that reduces to a matrix Chernoff bound \citep{SpielmanSrivastava:2011}. This is a critical component for previously known active learning guarantees for fitting linear functions using leverage score sampling \citep{Sarlos:2006}. 

\section{Main Result}
With preliminaries in place, we are ready to prove Theorem \ref{thm:main_result}. We begin with pseudocode for Algorithm \ref{alg:active-lean}, which obtains the guarantee of the theorem via leverage score sampling combined with a novel regularization strategy. 

\begin{algorithm}
	\caption{Leverage Score Based Active Learning for Single Neuron Models}
	\label{alg:active-lean}
	\raggedright
	{\bfseries input}: Matrix $\X\in \R^{n\times d}$, $L$-Lipschitz non-linearity $f: \R\rightarrow \R$ with $f(0)=0$,  query access to target vector $\y\in \R^n$, number of samples $m$.\\
	{\bfseries output}: Approximate solution to $\min_{\w}\|f(\X\w) - \y\|_2^2$.\\
	\begin{algorithmic}[1]
		\STATE Compute $\tau_i(\X)$ for all $i = 1, \ldots, n$.\\
		\STATE Set $p_i = \tau_i(\X) /\sum_{j=1}^n \tau_i(\X)$ for all $i = 1, \ldots, n$.\\
		\STATE Construct $\S\in \R^{m \times n}$ according to Definition \ref{def:imp_samp_matrix} with probabilities $p_1, \ldots, p_n$. 
		\STATE Query $\y$ at $m$ locations to obtain $\S\y \in \R^m$. 
		\STATE Solve the minimization problem: \vspace{-.5em}
		\begin{align*}\hat{\w} = \argmin_{\w: \|\S\X\w\|_2^2 \leq \frac{1}{\eps\cdot L^2}\|\S (\y\|_2^2} \|\S f(\X\w) -\S \y\|_2^2\end{align*}
		\vspace{-1em}
		\STATE {\bfseries return} $\hat{\w}$ 
	\end{algorithmic}
\end{algorithm}

The core of Algorithm~\ref{alg:active-lean} is the optimization problem:
\begin{align}
\hat{\w} = \argmin_{\w: \|\S\X\w\|_2^2 \leq \frac{1}{\eps\cdot L^2}\|\S (\y\|_2^2} \|\S f(\X\w) -\S \y\|_2^2 \, .
\label{eq:contrained_sketch_and_solve} 
\end{align}
Since this problem involves a Euclidean constraint on $\S\X\w$, it is notably different from the more standard (weighted) empirical risk minimization problem: $\min_{\w} \|\S f(\X\w) -\S \y\|_2^2$. We believe that the norm constraint is necessary for getting acceptable upper bounds in the agnostic setting and cannot be eliminated. However, in our experiments (Section \ref{sec:exp}) we were able to safely ignore the constraint without hurting empirical performance. 
In any case, with or without constraint, minimizing \eqref{eq:contrained_sketch_and_solve} is a non-convex neuron fitting problem, and we do not attempt to theoretically analyze its computational complexity in this paper; however, it can be solved easily in practice using standard first-order optimization methods (such as gradient descent or its projected version). 
%

As an first step to proving Theorem \ref{thm:main_result}, we link the quality of the solution to \eqref{eq:contrained_sketch_and_solve} to that of the optimum regressor, $\w^* \in \argmin_{\w} \|f(\X\w) - \y\|_2$, as follows. 
\begin{claim}\label{claim:trade_off}
Let $\hat{\w}$ be the vector returned by Algorithm \ref{alg:active-lean} and let $OPT = \|f(\X\w^*) - \y\|_2^2$. 
	With probability $49/50$, for a fixed constant $C > 0$, we have
	\begin{align*}
		\|\S f(\X\hat{\w}) -\S \y\|_2^2 \leq C\cdot \left(OPT + \eps L^2 \|\X\w^*\|_2^2\right).
	\end{align*}
\end{claim}
\begin{proof}
	Consider the case when $\|\S\X\w^*\|_2^2 \leq \frac{1}{\eps L^2}\|\S \y\|_2^2$. Then $\w^*$ satisfies the constraint of the above optimization problem so we have that $$\|\S f(\X\hat{\w}) -\S \y\|_2^2\leq \|\S f(\X\w^*) -\S \y\|_2^2 \leq C\cdot OPT.$$ The last inequality follows with probability $49/50$ via Markov's inequality since $$\E\left[\|\S f(\X\w^*) -\S \y\|_2^2\right] = \| f(\X\w^*) -\y\|_2^2 = OPT.$$
 
 On the other hand, consider the case where $\|\S\X\w^*\|_2^2 \ge \frac{1}{\eps L^2}\|\S \y\|_2^2$. Then we have that $\|\S\y\|_2^2 \leq \eps L^2\cdot  \|\S \X\w^*\|_2^2$. In this case, we can plug in the zero vector to the above minimization problem (since zero clearly satisfies the constraint) and conclude again that:
	\begin{align*}
		\|\S f(\X\hat{\w}) -\S \y\|_2^2 &\leq \|\S f(\X\mathbf{0}) - \S \y\|_2^2 = \|\S \y\|_2^2\\
  &\leq \eps L^2\|\S \X\w^*\|_2^2 \leq 2\eps L^2 \|\X\w^*\|_2^2.
	\end{align*}
The last inequality follows from the subspace embedding inequality from Lemma \ref{lem:subspace}. Note that the constraint of $f(0) = 0$ is used above as
$f(\X\mathbf{0}) = f(\mathbf{0}) = \mathbf{0}$. 
\end{proof}

\subsection{Concentration Bounds}
Claim \ref{claim:trade_off} upper bounds the error of $\hat{\w}$ in solving the \emph{subsampled} regression problem $\|\S f(\X\hat{\w}) -\S \y\|_2^2$. To show that it also provides a good solution for the original problem $\min_{\w}\|f(\X\hat{\w}) - \y\|_2^2$ we require several concentration results that are a consequence of leverage score sampling. These results are similar to Lemma \ref{lem:subspace}, except that they show that sampling with $\S$ also preserves the norm of vectors obtained via \emph{non-linear} transformations of the form $f(\X \w)$.

The first bound gives a guarantee on preserving the distance between two fixed vectors, $f(\X\w_1) - f(\X\w_2)$. In contrast to the relative error subspace embedding of Lemma \ref{lem:subspace}, the bound involved an additive error term; this extra term is likely unavoidable in the most general case of Lipschitz $f$. 

\begin{lemma}\label{prop:bern}
	Let $f$ and $\X$ be as in Theorem \ref{thm:main_result}. Let $\S\in \R^{m\times n}$ be an {importance sampling matrix} chosen with probabilities $\{p_1, \ldots, p_n\}$, where $p_i = \tau_i(\X)/\sum_{i=1}^n \tau_i(\X)$. As long as $m \geq \frac{3d\log(2/\delta)}{\eps^2}$, then with probability $\ge 1-\delta$, for any fixed pair of vectors $\w_1, \w_2 \in \R^d$, we have:  
	\begin{align*}
		\norm{f(\X\w_1) - f(\X\w_2)}_2^2 - \eps L^2 \norm{\X\w_1 - \X\w_2}_2^2 &\leq   \norm{\S f(\X\w_1) - \S f(\X\w_2)}_2^2 \\ &\leq \norm{f(\X\w_1) - f(\X\w_2)}_2^2 + \eps L^2 \norm{\X\w_1 - \X\w_2}_2^2.
	\end{align*}
\end{lemma}
By combining Lemma \ref{prop:bern} with an $\eps$-net argument, we can extend the bound to obtain a one-sided guarantee that involves the distance between $f(\X\w^*)$ and $f(\X\w)$ for  \emph{all} $\w$ within a ball of radius $R$. 
\begin{lemma}\label{lemm:net_construction}
	Let $f$, $\X$, and $\y$ be as in Theorem \ref{thm:main_result}. Let $\w^* = \argmin_\w \|f(\X\w) - \y\|_2^2$, $R$ be a fixed radius and $C_1,C_2>0$ be fixed constants. Let $\S\in \R^{m\times n}$ be an {importance sampling matrix} as in Lemma \ref{prop:bern}. As long as $m \geq  c \frac{d^2\log(1/\eps)}{\eps^4}$ for $\eps < 1$ and a fixed constant $c$ then with probability $49/50$, for  all $\wh \in \mathcal{B}^d(R)$,
	\begin{align*}
		\norm{f(\X\wh) - f(\X\w^*)}_2^2 &\le 4 \cdot \norm{\S f(\X\wh) - \S f(\X\w^*)}_2^2 
  +  \eps^2L^2 R^2  +  \eps^2 L^2 \|\X\w^*\|_2^2.
	\end{align*}
\end{lemma}

Proofs of Lemma \ref{prop:bern} and \ref{lemm:net_construction} are deferred to Appendix \ref{sec:app}.

\subsection{Proof of Main Result}
With Claim \ref{claim:trade_off} and Lemma \ref{lemm:net_construction} in place, we are now ready to prove our main result.
\begin{proof}[Proof of Theorem \ref{thm:main_result}.]
	First note that, without loss of generality, we can assume that $\X$ has orthonormal columns. In particular, if $\X$ is not orthonormal, we can write it as $\X = \Q\mathbf{R}$ where $\Q \in \R^{n\times \rank(\X)}$ has orthonormal columns and $\mathbf{R}$ is a square full-rank matrix. The leverage scores of $\Q$ are equal to those of $\X$. Moreover, 
	any solution $\wh$ to \eqref{eq:contrained_sketch_and_solve} has a corresponding solution $\mathbf{R}\wh$ to the minimization problem if $\X$ were replaced by $\Q$. So solving the above problem is equivalent to first explicitly orthogonalizing $\X$ and solving the same problem. 
	
	Next, we use the elementary fact that for any vectors $\mathbf{a}$ and $\mathbf{b}$, 
 $\|\mathbf{a} + \mathbf{b}\|_2^2 \leq 2 \|\mathbf{a}\|_2^2 + 2 \|\mathbf{b}\|_2^2.$
 This give the bound:
	\begin{align}
		\label{eq:main_first_step}
		 \norm{f(\X\wh) - \y}_2^2 \nonumber
  &\le 2\norm{f(\X\wh) - f(\X\w^*)}_2^2 + 2\norm{f(\X\w^*) - \y}_2^2 \nonumber\\
		&\le 2\norm{f(\X\wh) - f(\X\w^*)}_2^2 + 2\cdot OPT.
	\end{align}
	We need to bound the first term. To do so, we first observe that, thanks to the constraint imposed in \eqref{eq:contrained_sketch_and_solve}, the norm of $\wh$ can be bounded, which allows us to apply Lemma \ref{lemm:net_construction}. In particular, we claim that with probability $49/50$,
	\begin{align}
		\label{eq:wh_norm_bound}
\|\wh\|_2^2 \leq \frac{100}{\eps L^2}\cdot \|\y\|_2^2.
	\end{align}
	
To see that this is the case, note that under our assumption that $\X$ is orthogonal, we have $\|\wh\|_2^2 = \|\X\wh\|_2^2$. We can bound $\|\X\wh\|_2^2$ as follows:
	\begin{align*}
		\norm{\X\wh}_2^2 &\le 2\norm{\S\X\wh}_2^2 \quad (\text{Lemma \ref{lem:subspace}})\\
		&\le 2\frac{1}{\eps \cdot L^2}\norm{\S\y}_2^2 \quad (\text{From the constraint in }  \eqref{eq:contrained_sketch_and_solve})\\
		&\le \frac{100}{\eps \cdot L^2}\norm{\y}_2^2 \quad (\text{Markov's inequality).}
	\end{align*}
In the last inequality, we used that $\E[\|\S\y\|_2^2] = \|\y\|_2^2$.

Since $\wh$ lies in $\mathcal{B}^d(R)$, where $R^2 = \frac{100}{\eps L^2}\cdot \|\y\|_2^2$, we can apply Lemma \ref{lemm:net_construction} along with Markov's inequality to conclude that, as long as $m \geq  c \frac{d^2\log(1/\eps)}{\eps^4}$:
	\begin{equation}
 \begin{aligned}
		\norm{f(\X\wh) - f(\X\w^*)}_2^2 
   & \leq 2 \norm{\S f(\X\wh) - \S f(\X\w^*)}_2^2 + 100\eps \|\y\|_2^2 + \eps^2 L^2 \norm{\X\w^*}_2^2\\
		&\leq 4 \norm{\S f(\X\wh) - \S\y}_2^2  + 4 \norm{\S f(\X\w^*) - \S\y}_2^2
  + 100\eps \|\y\|_2^2 + \eps^2 L^2 \norm{\X\w^*}_2^2 \nonumber \\
		&\le 4 \norm{\S f(\X\wh) - \S\y}_2^2  + 50 \cdot OPT  + 100\eps \|\y\|_2^2 + \eps^2 L^2 \norm{\X\w^*}_2^2. \nonumber
\end{aligned}
        \end{equation}
As in the proof of Claim \ref{claim:trade_off}, the last inequality follows with probability $49/50$ via Markov's inequality since $\E\left[\|\S f(\X\hat{\w}) -\S \y\|_2^2\right] = \| f(\X\w^*) -\y\|_2^2 = OPT$.

	Next, we apply Claim \ref{claim:trade_off} to bound $\norm{\S f(\X\wh) - \S\y}_2^2  \leq O\left(OPT + \eps L^2 \|\X\w^*\|_2^2\right)$. So overall, we conclude that for a constant $C$,
	\begin{align}
		\label{eq:last_step_basically}
		\|f(\X\hat{\w}) - \y\|_2^2 &\leq C \cdot \left(OPT + \eps L^2 \|\X\w^*\|_2^2+ \eps \|\y\|_2^2 \right). 
	\end{align}
By triangle inequality, we have that 
\begin{align*}
\|\y\|_2^2 &\leq 2OPT + 2\|f(\X\w^*)\|_2^2 \\
&\leq 2OPT + 2L^2\|\X\w^*\|_2^2.
\end{align*}
Using this fact, plugging \eqref{eq:last_step_basically} into \eqref{eq:main_first_step}, and rearranging terms yields the stated main result with probability. Union bounding overall all events assumed to hold in the proof, the result holds with probability $> 45/50 = 9/10$.
\end{proof}

\section{Experimental Results} 
\label{sec:exp} 

\begin{figure*}[t]
	\begin{subfigure}[t]{0.34\columnwidth}
\begin{tikzpicture}[scale = 0.67]

\definecolor{darkgray176}{RGB}{176,176,176}
\definecolor{green111723}{RGB}{1,117,23}
\definecolor{lightgray204}{RGB}{204,204,204}
\definecolor{navy028127}{RGB}{0,28,127}

\begin{axis}[
style = very thick,
legend cell align={left},
legend style={fill opacity=1, draw opacity=1, text opacity=1, draw=lightgray204,  nodes={scale=0.8, transform shape}},
grid style={line width=.1pt, draw=gray!20},
major grid style={line width=.1pt,draw=gray!20},
tick align=outside,
tick label style={/pgf/number format/fixed},
tick pos=left,
title={},
x grid style={darkgray176},
xlabel={Sample Sizes},
xmajorgrids,
xmin=20.25, xmax=124.75,
xtick style={color=black},
y grid style={darkgray176},
ylabel={Relative Error},
ymajorgrids,
ymin=0.0023571586087109, ymax=0.236669129951777,
ytick style={color=black}
]

\addplot [ultra thick, blue, dashed]
table {%
25 0.100925192973117
40 0.0446895562657606
60 0.0283669392366791
80 0.0181271816077063
100 0.0163842432900762
120 0.0130077027606685
};
\addlegendentry{Leverage score sampling}
\addplot [ultra thick, red, dotted]
table {%
25 0.22601858579982
40 0.0829142346512896
60 0.0396788361775591
80 0.0270738970960071
100 0.022419913252031
120 0.0225479424385879
};
\addlegendentry{Uniform sampling}
\end{axis}

\end{tikzpicture}
		\caption{$\relu(.4 x_1 + .4 x_2 -.4)$, uniform data}
	\end{subfigure}
	\begin{subfigure}[t]{0.32\columnwidth}

\begin{tikzpicture}[scale = 0.67]

\definecolor{darkgray176}{RGB}{176,176,176}
\definecolor{green111723}{RGB}{1,117,23}
\definecolor{lightgray204}{RGB}{204,204,204}
\definecolor{navy028127}{RGB}{0,28,127}

\begin{axis}[
style = very thick,
legend cell align={left},
legend style={fill opacity=1, draw opacity=1, text opacity=1, draw=lightgray204,  nodes={scale=0.8, transform shape}},
tick align=outside,
tick pos=left,
y tick label style={
                    /pgf/number format/.cd,
                    fixed relative,
                    },
title={},
grid style={line width=.1pt, draw=gray!20},
major grid style={line width=.1pt,draw=gray!20},
x grid style={darkgray176},
xlabel={Sample Sizes},
xmajorgrids,
xmin=7.5, xmax=62.5,
xtick style={color=black},
y grid style={darkgray176},
ylabel={},
ymajorgrids,
ymin=0.00261204948162059, ymax=0.0095381417121235,
ytick style={color=black},
]

\addplot [ultra thick, blue, dashed]
table {%
10 0.00597962557568885
20 0.00404941364163921
40 0.00300233459293011
50 0.00300233459293011
60 0.00292687185573436
};
\addlegendentry{Leverage score sampling}
\addplot [ultra thick, red, dotted]
table {%
10 0.00922331933800973
20 0.00469915742453822
40 0.0034002763146532
50 0.00327380958771711
60 0.00300233459293011
};
\addlegendentry{Uniform sampling}
\end{axis}

\end{tikzpicture}
		\caption{$(.8 x_1 + .1 x_2 -.1)^2$, uniform data}
	\end{subfigure}
	\begin{subfigure}[t]{0.32\columnwidth}
\begin{tikzpicture}[scale = 0.67]

\definecolor{darkgray176}{RGB}{176,176,176}
\definecolor{green111723}{RGB}{1,117,23}
\definecolor{lightgray204}{RGB}{204,204,204}
\definecolor{navy028127}{RGB}{0,28,127}

\begin{axis}[
style = very thick,
legend cell align={left},
legend style={fill opacity=1, draw opacity=1, text opacity=1, draw=lightgray204,nodes={scale=0.8, transform shape}},
tick align=outside,
tick pos=left,
tick label style={/pgf/number format/fixed},
title={},
x grid style={darkgray176},
xlabel={Sample Sizes},
xmajorgrids,
xmin=7.5, xmax=62.5,
xtick style={color=black},
y grid style={darkgray176},
ylabel={},
ymajorgrids,
ymin=-0.00154717183531562, ymax=0.0485712581645178,
ytick style={color=black}
]
\addplot [ultra thick, blue, dashed]
table {%
10 0.0152941842618617
20 0.00346075410938869
30 0.00197071898393941
40 0.00198739979550974
50 0.00151357874697398
60 0.00073093861922226
};
\addlegendentry{Leverage score sampling}
\addplot [ultra thick, red, dotted]
table {%
10 0.0462931477099799
20 0.0076300300543609
30 0.00388824118326503
40 0.00239731001001316
50 0.00195835134227851
60 0.00199501534690489
};
\addlegendentry{Uniform sampling}
\end{axis}

\end{tikzpicture}
		\caption{$\relu(.7 x_1 + .1 x_2 -.4)$, Gaussian data}
	\end{subfigure}
	\caption{Median relative error for learning two-dimensional single neuron models involving a ReLU non-linearity for a synthetic $\X$ with rows selected from a uniform or Gaussian distribution. The target vector $\y$ was obtained by corrupting the ground truth with Gaussian noise with variance $.05$. In all cases our active leverage score sampling method outperforms naive uniform sampling. As expected, the improvement is more significant when a small number of samples are taken.}
	\label{fig:synth}
\end{figure*}
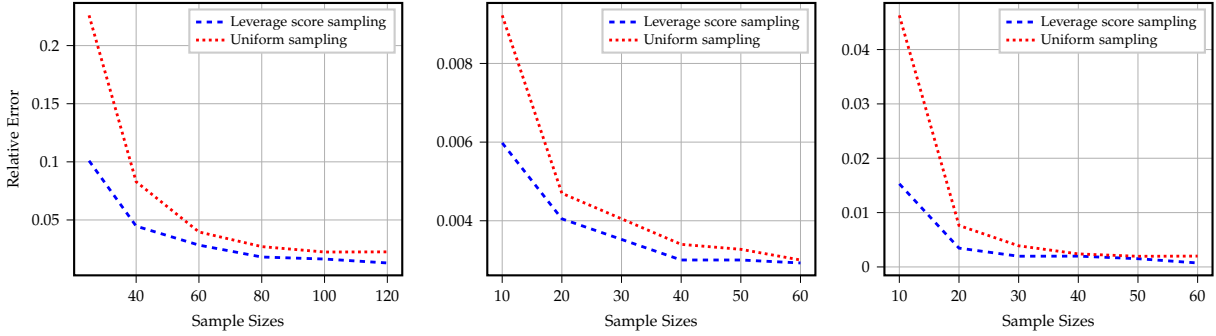

To complement our theoretical analysis, we also provide experimental results showing the promise of leverage score sampling for actively learning singe neuron models. 
We consider both synthetic data problems, as well as several tests derived from differential equation approximation problems. We focus on sample efficiency -- i.e.,  how many samples from $\y$ are required to obtain a good approximation to $\min_{\w}\|f(\X\w) - \y\|_2^2$. Computational efficiency is not a major concern: as discussed, in typical applications of single-neuron learning in computational science, collecting samples requires numerically solving a differential equation, which dominates any runtime cost of the actual fitting procedure \citep{AdcockBrugiapagliaDexter:2022}. Moreover, leverage score sampling has already proven an efficient active learning tool for linear function classes \citep{CohenDeVore:2015}. 

Overall, for all problems tested, our experiments show that leverage score sampling obtains a better sample/accuracy trade-off than the standard approach of choosing sample points uniformly at random from $\X$. 

\noindent\textbf{Synthetic Data.} For the synthetic data problems, we set $\X$ to contain $10^5$ random vectors drawn from either a two dimensional Gaussian distribution (``Gaussian data''), or uniformly from the two-dimensional box $[-1,1]^2$ (``uniform data''). We also add a column of all $1$'s to $\X$, which corresponds to including a bias term in the single-neuron model. We select a ground truth $\w^*$, and create $\y = f(\X\w^*) + \g$, where $\g$ is a vector of mean-centered Gaussian noise and $f$ is the $\relu$ non-linearity. We then compute $\wh$ by subsampling data via leverage scores (as in Algorithm \ref{alg:active-lean}) and minimizing $\|\S f(\X\w) - \S\y\|_2^2$ over our subsampled data.

In our experiments we found that the constraint in Eq.~\ref{eq:contrained_sketch_and_solve} could be dropped without hurting the performance of leverage score sampling. For these low-dimensional synthetic problems, we simply used brute force search to optimize weights to ensure that a true minimum was found. 
 We then run 100 trials each for various subsample sizes, and report median relative error: ${\|f(\X\w^*) - \y\|_2^2}/{\|\y\|_2^2}$.  As show in Figure \ref{fig:synth}, leverage scores sampling outperforms uniform sampling, especially for a relatively small number of samples. As expected, for a large number of samples, both methods eventually perform comparably, as both will obtain a $\wh$ very close to the optimal $\w^*$.

\noindent\textbf{Test Problems.} We consider three test problems involving the approximation of various Quantities of Interest (QoI's) for three parametric differential equations: a damped harmonic oscillator, the heat equation, and the steady viscous Burger's equation. 
	\begin{figure*}[!htb]
	\centering
	\begin{subfigure}[t]{0.33\columnwidth}
		\includegraphics[width=.9\columnwidth]{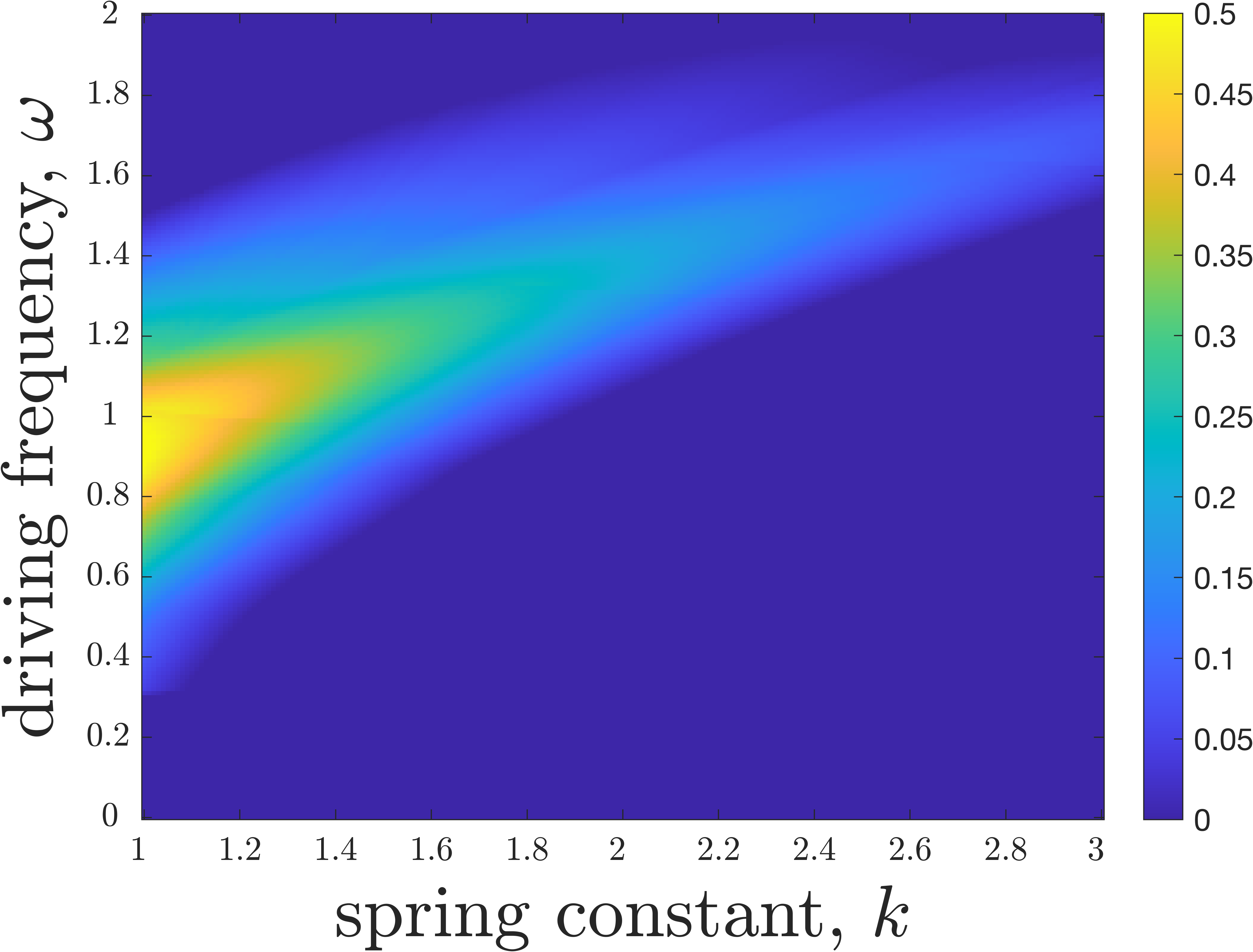}
		\caption{True Quantity of Interest.}
	\end{subfigure}\hfill
	\begin{subfigure}[t]{0.33\columnwidth}
		\includegraphics[width=.9\columnwidth]{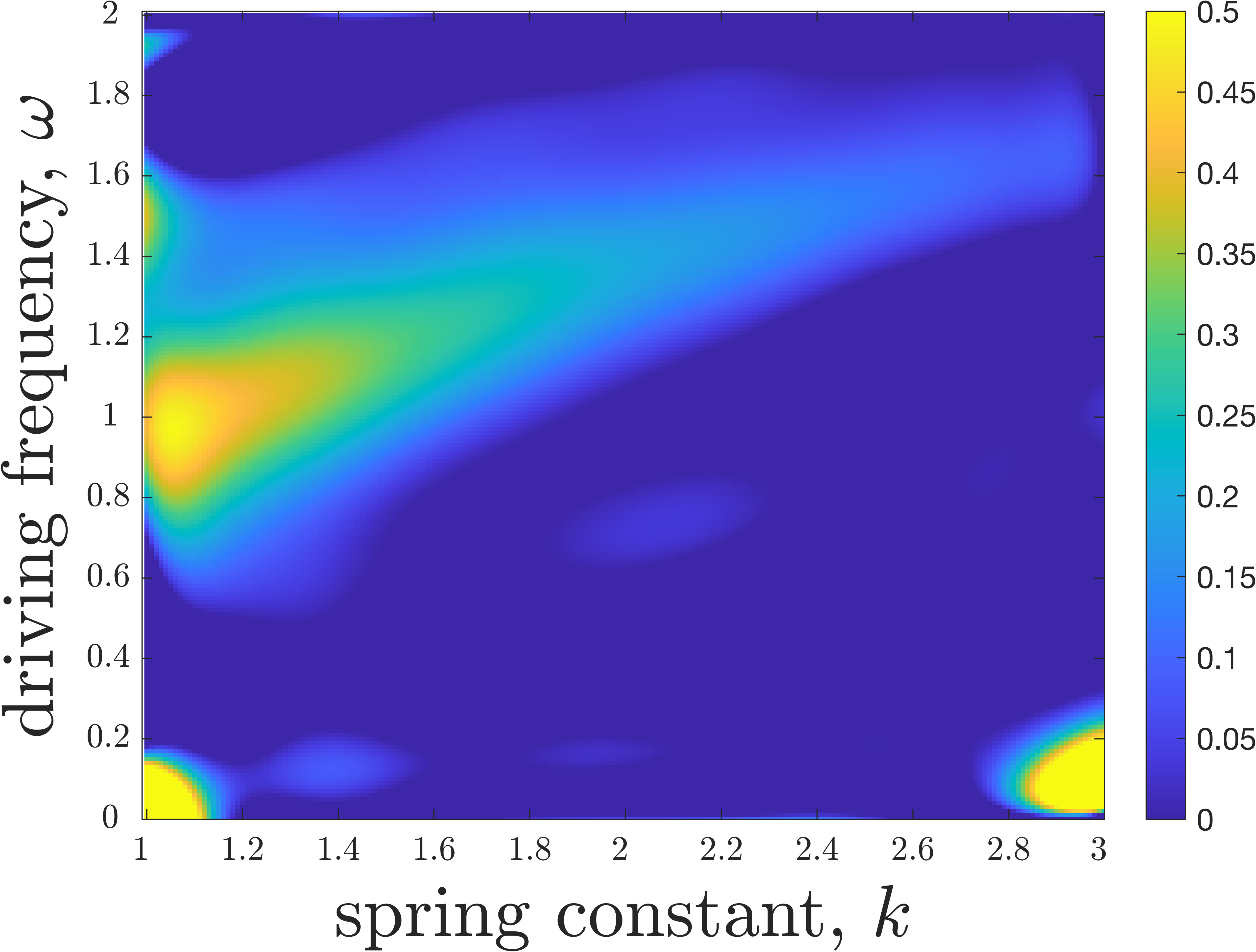}
		\caption{Approx. based on $200$ uniform samples.}
	\end{subfigure}\hfill
	\begin{subfigure}[t]{0.33\columnwidth}
		\includegraphics[width=.9\columnwidth]{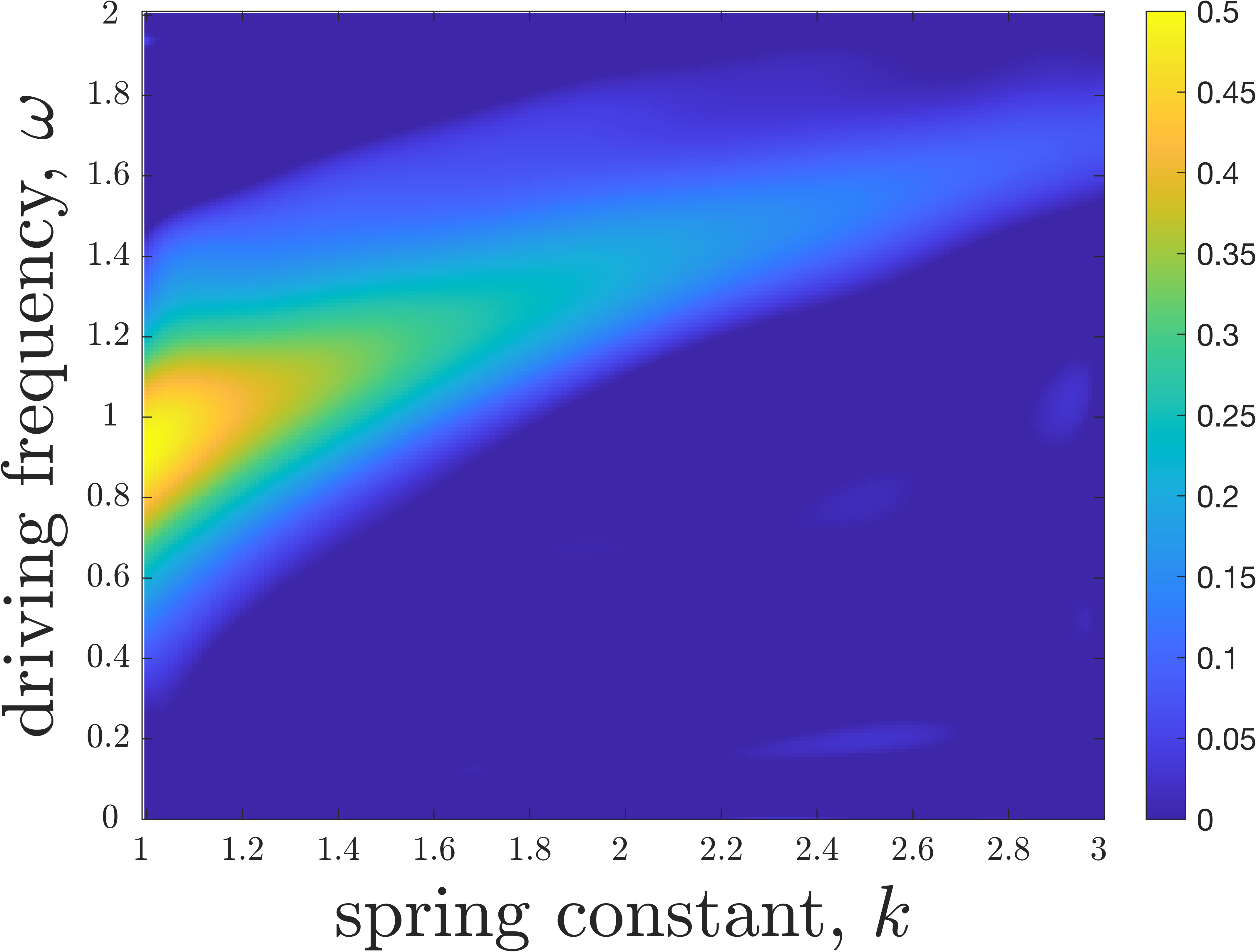}
		\caption{Approx. based on $200$  leverage samples.}
	\end{subfigure}
        \begin{subfigure}[t]{0.33\columnwidth}
		\includegraphics[width=.9\columnwidth]{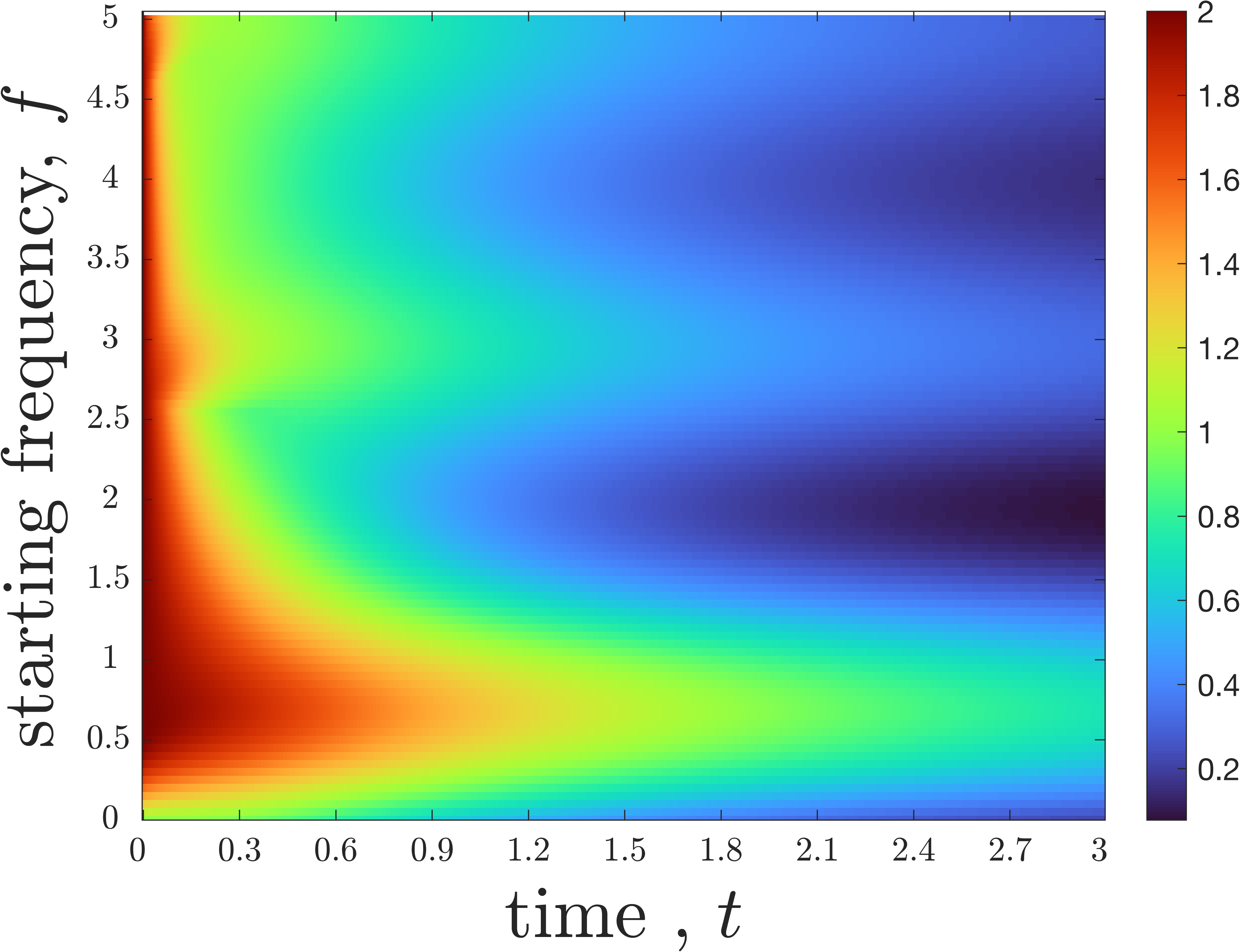}
		\caption{True Quantity of Interest.}
	\end{subfigure}\hfill
	\begin{subfigure}[t]{0.33\columnwidth}
		\includegraphics[width=.9\columnwidth]{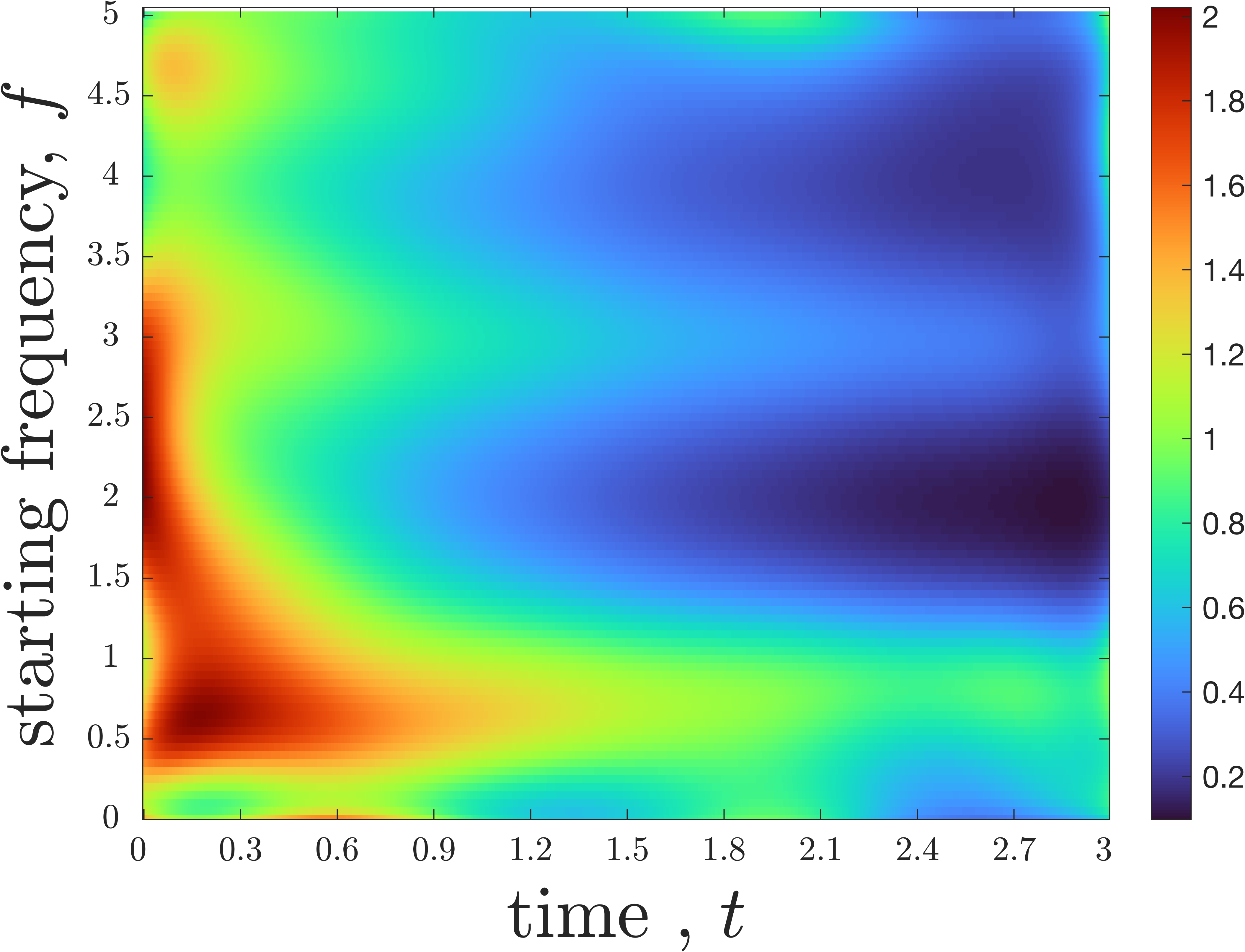}
		\caption{Approx. based on $120$ uniform samples.}
	\end{subfigure}\hfill
	\begin{subfigure}[t]{0.33\columnwidth}
		\includegraphics[width=.9\columnwidth]{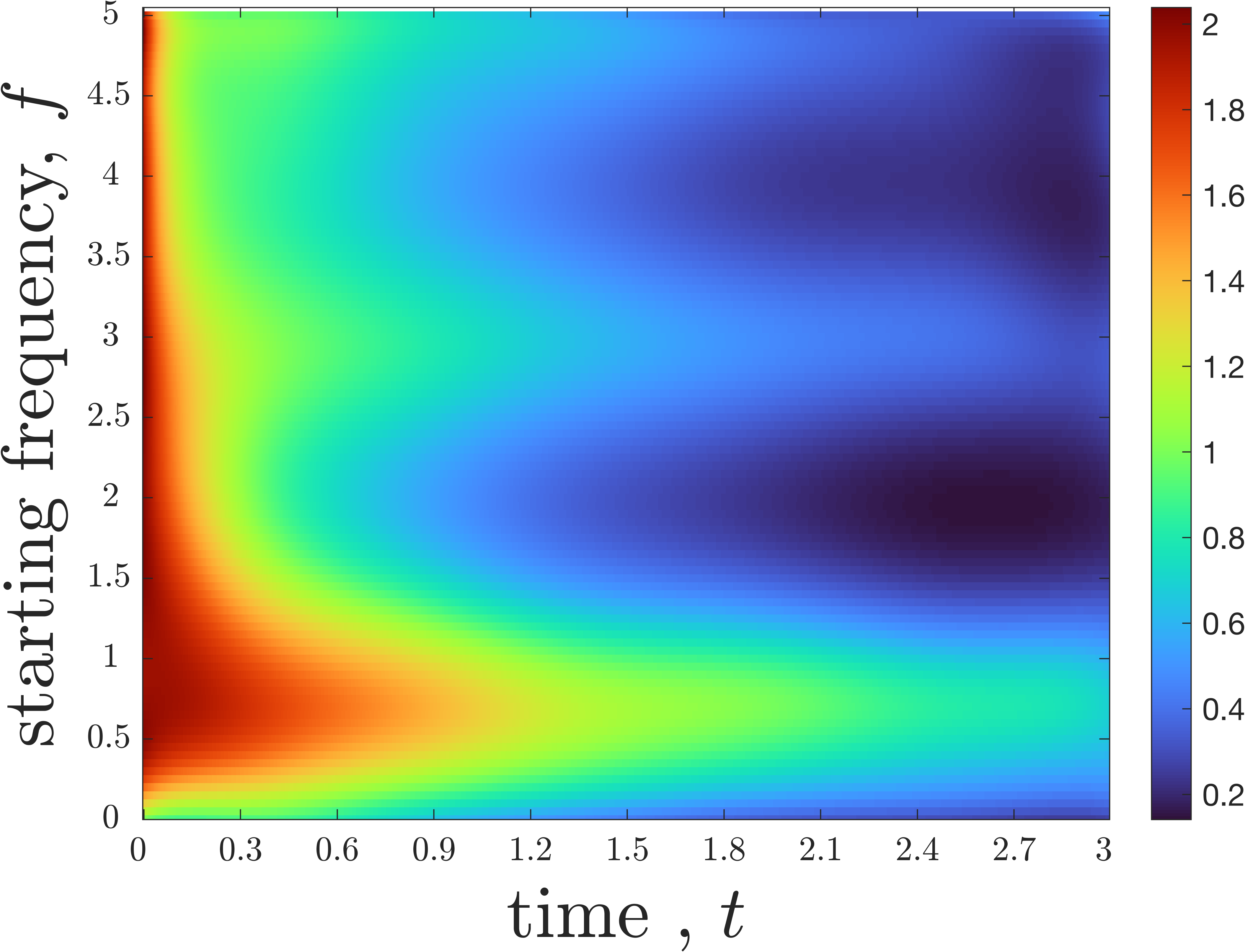}
		\caption{Approx. based on $120$ leverage samples.}
	\end{subfigure}
	\caption{
		The top $3$ images represent a plot of single neuron model fit to the maximum displacement QoI for a parametric ODE modeling a driven harmonic oscillator; and, the bottom $3$ images represent the fit of the maximum temperature QoI for the heat equation PDE with sinusoidal initial condition. Evidently, leverage score sampling  provides a better fit.}


\label{fig:qualitative_eq}
\end{figure*}

\textbf{Test 1. } We first consider a second order ODE modeling a damped harmonic oscillator
with a sinusoidal force applied, which corresponds to the parametric differential equation:
\begin{align*}
	\frac{d^2x}{dt^2}(t) + c \cdot \frac{dx}{dt}(t) + k\cdot x(t) &= f\cdot\cos(\omega t);\\  
	x(0) = x_0,\,\,\, 
	\frac{dy}{dt}(0) &= x_1.\nonumber
\end{align*}
Here, $(x,t)$ is the oscillator's space and time coordinates, and $c,k,f,\omega$ are parameters. The choice of parameters significantly impact the final solution; for example, if the frequency term $\omega$ is close to the resonant frequency of the oscillator, we expect the driving force to lead to large oscillations. We consider as our QoI the maximum oscillator displacement after 20 seconds, and the goal is to estimate this value for all $k$ and $\omega$ in the rectangle $\mathcal{U} = [1,3]\times [0,2]$. 

We choose to approximate the QoI (which is always positive) with a function of the form $\relu(p(k,\omega))$, where $p$ is bivariate polynomial with total degree $q = 9$. This is accomplished by setting $\X$ to be a Vandermonde matrix of Legendre polynomials evaluated at a grid of values on $[1,3]\times [0,2]$. $\X$ has $55 = (q+1)(q+2)/2$ columns, which is the total number of terms in a degree $q$ bivariate polynomial.  Each row in $\X$ corresponds to a different choice of parameters $k,\omega$.
We fit our single neuron model to the QoI using gradient descent with a standard adaptive step-size, again dropping the constraint in \eqref{eq:contrained_sketch_and_solve}. As shown in the top three images of Figure \ref{fig:qualitative_eq}, for a fixed number of samples, leverage score sampling leads to a visually better fit than uniform sampling. Quantitatively, we see in Figure \ref{fig:quant} that leverage score sampling gives almost an order of magnitude lower error across a wide range of sampling numbers. 

In Figure \ref{fig:sampvis}, we visualize how, for this problem, uniform samples differ from those collected using leverage scores. The Vandermonde matrix $\X$ has higher leverage score for rows corresponding to points near the boundary of $[1,3]\times [0,2]$, so more samples are taken for $(k,\omega)$ values near the boundary. The benefits of sampling near the boundary are well-known for fitting simple polynomials \citep{CohenMigliorati:2017}. It is interesting that these benefits remain  when the polynomial is combined with a non-linearity. 
\begin{figure}[b!]
	\centering
	\begin{subfigure}{0.5\columnwidth}
		\centering
		\includegraphics[width=.9\columnwidth]{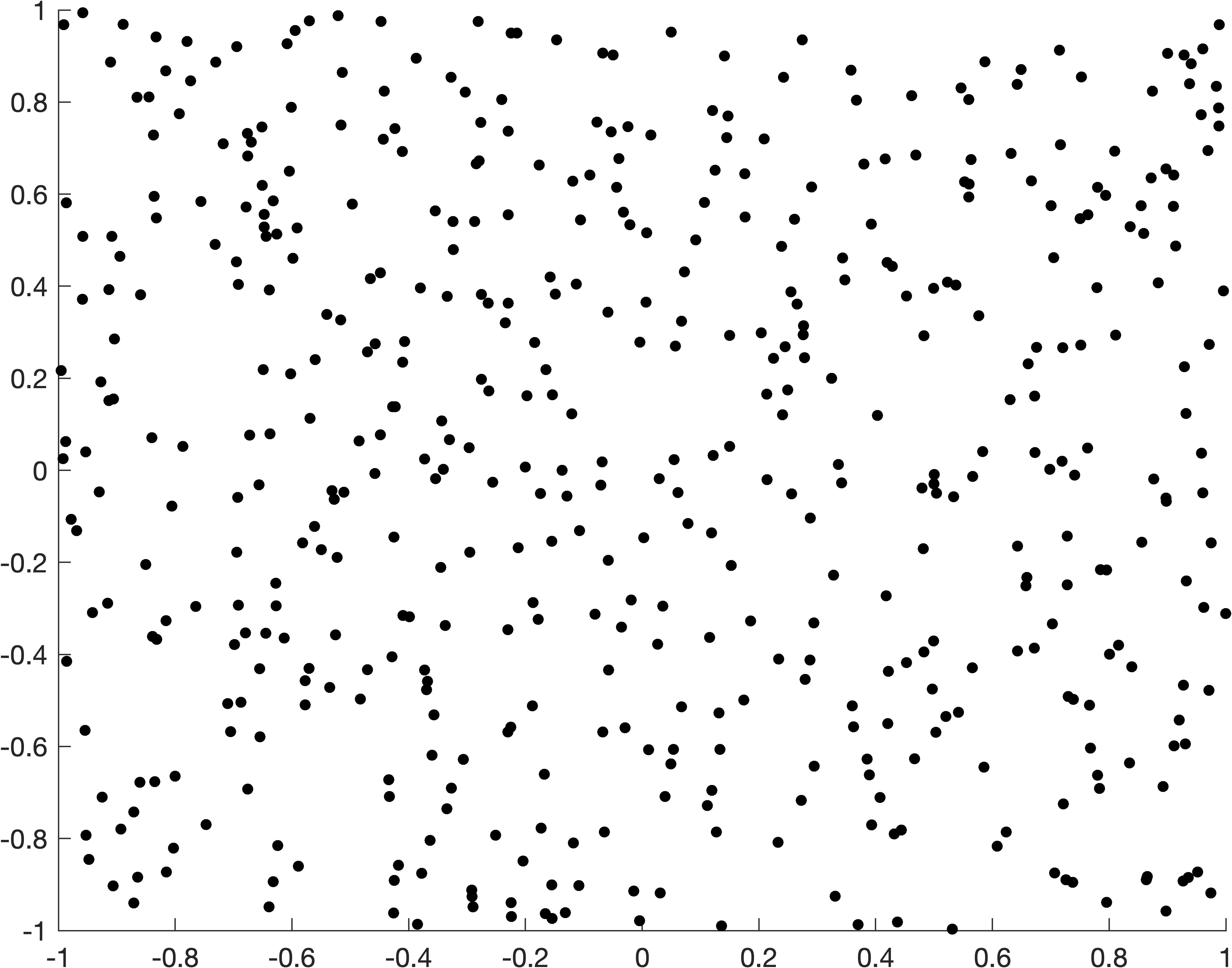}
		\caption{Uniform Random Samples}
	\end{subfigure}\hfill
	\begin{subfigure}{0.5\columnwidth}
		\centering
		\includegraphics[width=.9\columnwidth]{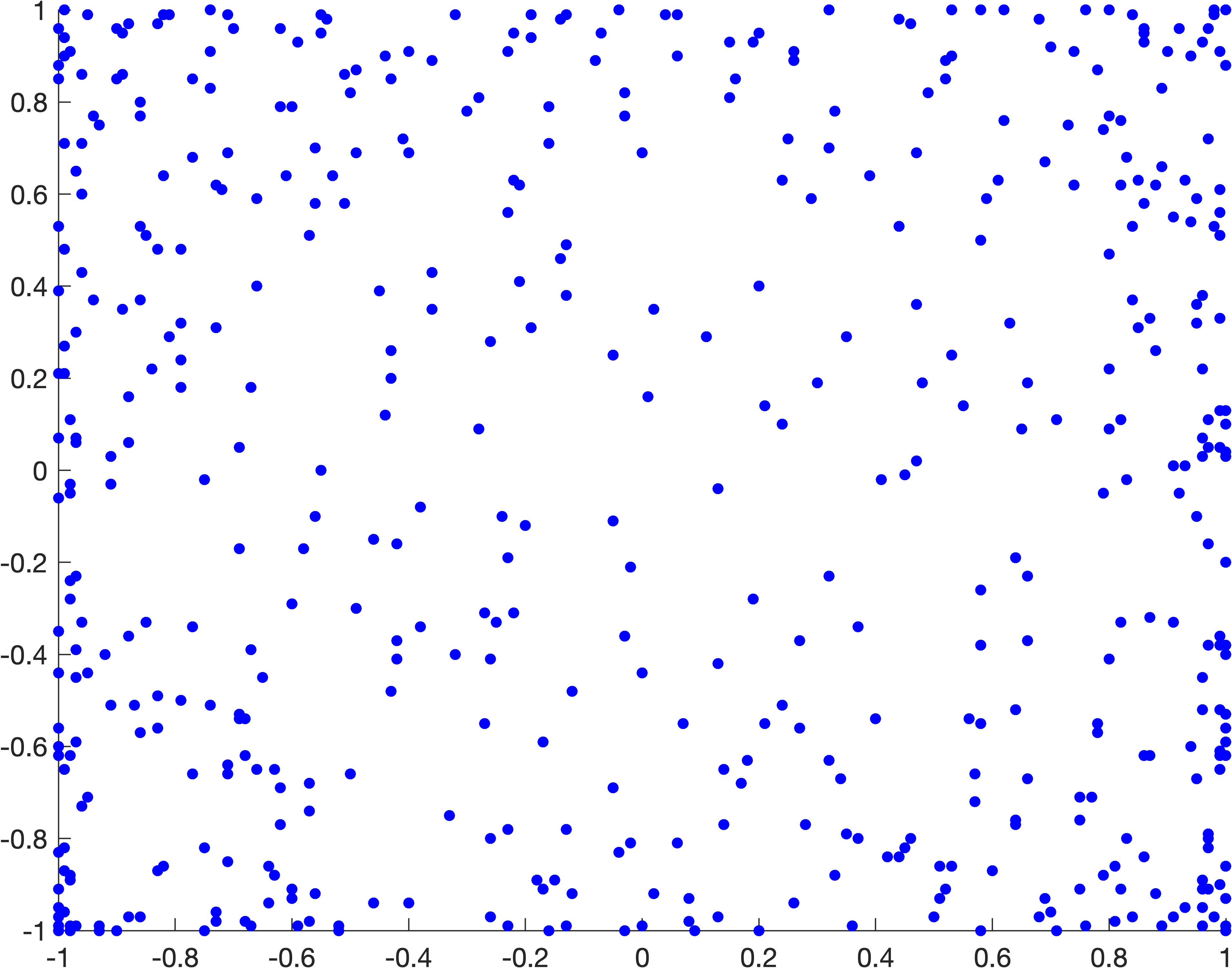}
		\caption{Leverage Score Samples}
	\end{subfigure}
	\caption{
		The plots visualize uniform vs. leverage score sampling for selecting example parameter  values from the box $[1,3]\times [0,2]$ for fitting the QoI for Test Problem 1. Our leverage score method tends to sample more heavily near the perimeter of the box to fit the single neuron model.
	}
	\label{fig:sampvis}
\end{figure}

\begin{figure*}[t!]
	\begin{subfigure}[t]{0.33\columnwidth}
		\includegraphics[width=\columnwidth]{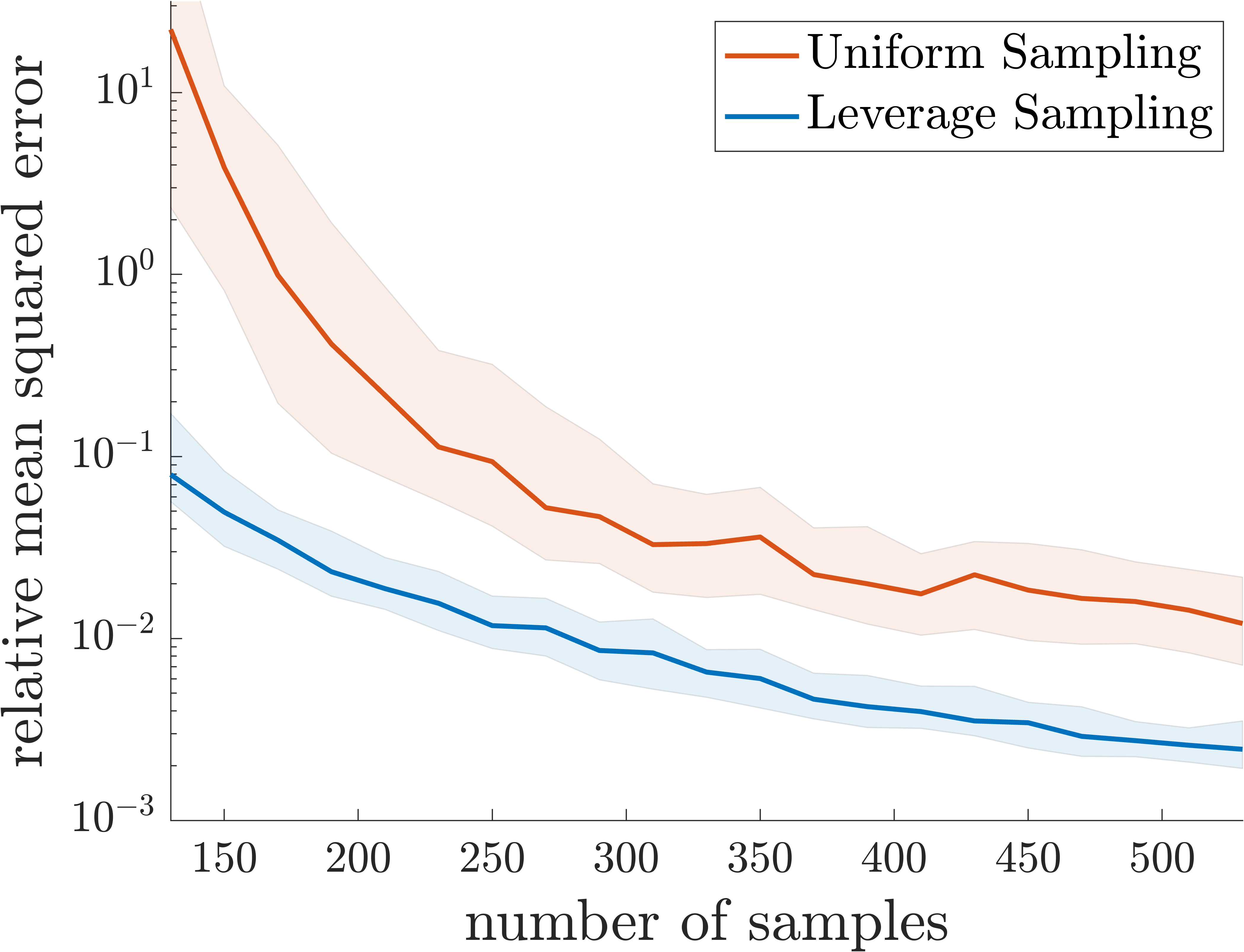}
		\caption{Damped harmonic oscillator.}
	\end{subfigure}
	\begin{subfigure}[t]{0.33\columnwidth}
	\includegraphics[width=\columnwidth]{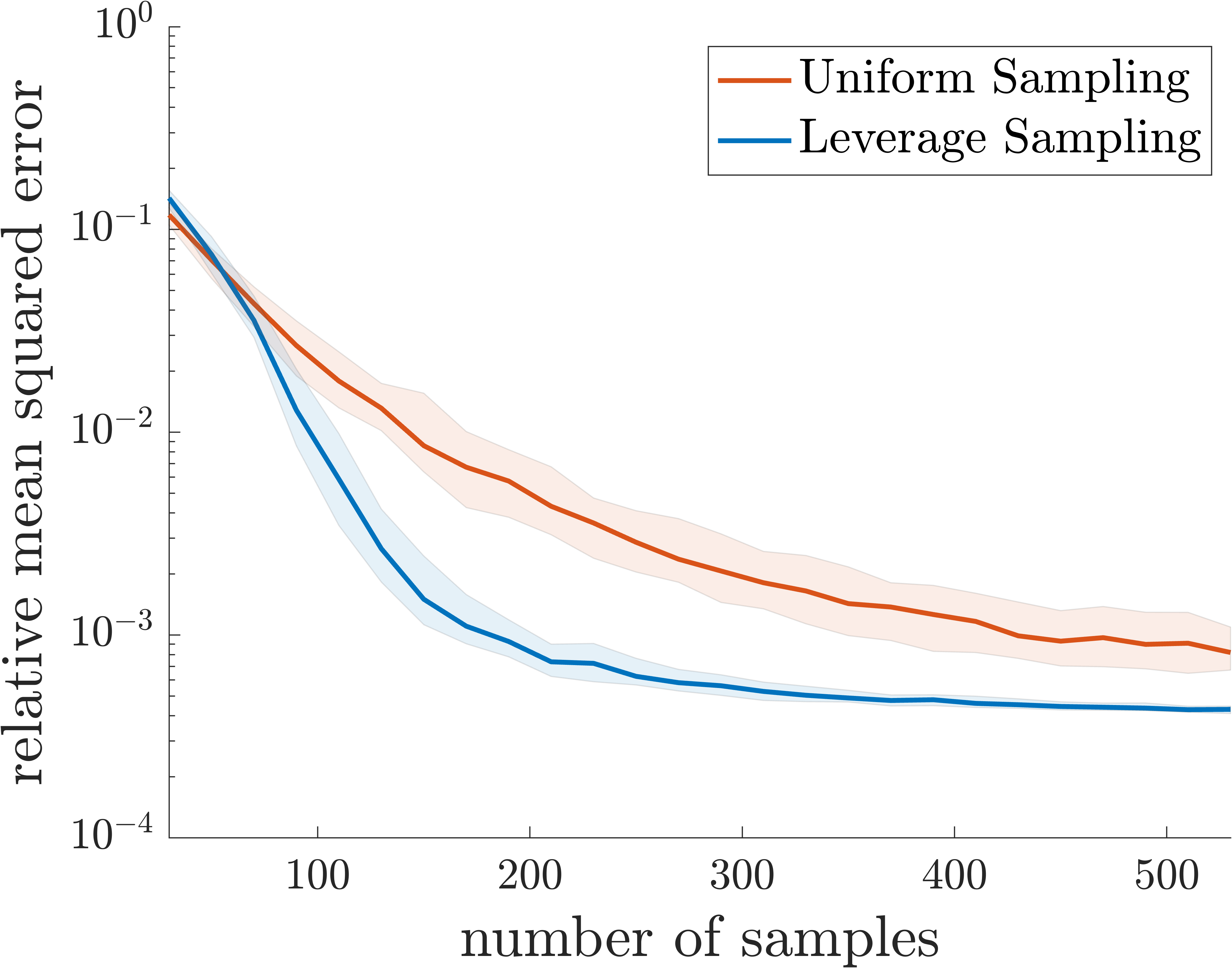}
	\caption{Heat equation.}
	\end{subfigure}
	\begin{subfigure}[t]{0.33\columnwidth}
		\includegraphics[width=\columnwidth]{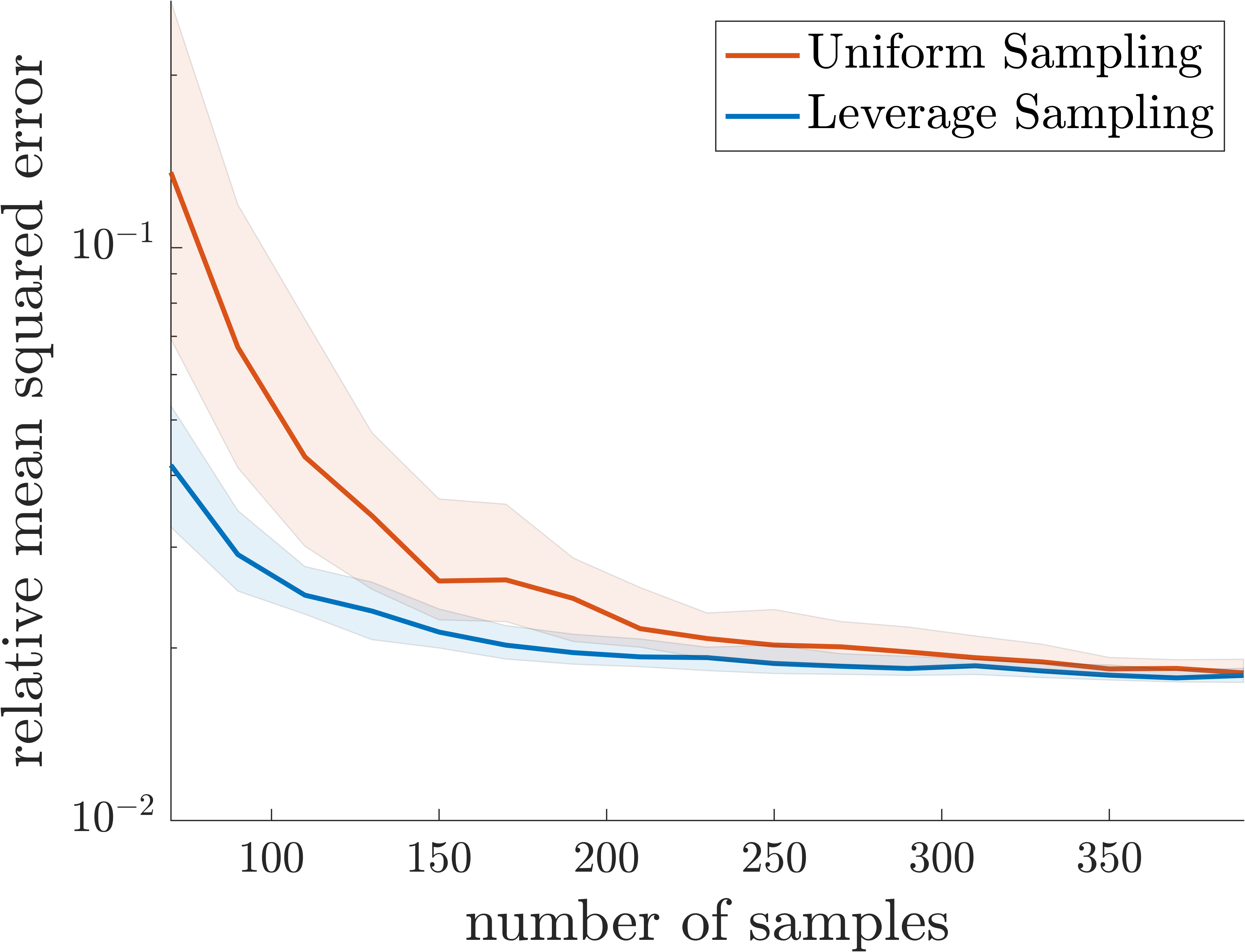}
		\caption{Steady viscous Burger's equation.}
	\end{subfigure}

	\caption{
		Sample complexity vs. relative error $\|f(\X\wh) - \y\|_2^2/\|\y\|_2^2$ for fitting the QoI's specified for our three test problems. All experiments were run for 100 trials per sample value and we plot the median error and interquartile range.
	}
	\label{fig:quant}
\end{figure*}

\textbf{Test 2.} We consider the 1-dimensional heat equation for values of $x\in [0,1]$ with a time-dependent boundary equation and sinusoidal initial condition. This is modeled by the partial differential equation:
\begin{align*}
	\pi \frac{\partial u}{\partial t} = \frac{\partial^2 u}{\partial x^2},\,\,\,
	\mu(0,t) &= 0,\,\,\,  \mu(x,0) = \sin(\omega \pi x)\\
	\pi e^{-t} + \frac{\partial u (1,t)}{\partial t} &= 0
\end{align*}
As our QoI we consider the maximum temperate over all values of $x$ for times $t\in [0,3]$ and frequencies $\omega \in [0,5]$. This leads to a highly varied QoI surface, which we again choose to fit with a model of the form $f(p(t,\omega))$.  We let $p$ be a degree $q = 11$ bivariate polynomial, so $\X$ has $78$ columns. For this problem we choose $f(a) = e^a$ to be the exponential function; otherwise the experimental setup is identical to Test Problem 1. Despite the fact that this non-linearity is \emph{not Lipschitz}, we again see visually better performances of leverage score sampling for a fixed number of samples in  the bottom three plots Figure \ref{fig:qualitative_eq}, and quantitatively better error in Figure \ref{fig:quant}. This result suggests our leverage score based active learning method may be robust to non-linearities that are just ``locally'' instead of globally Lipschitz. 

\textbf{Test 3. } Finally, we consider steady state viscous Burger's equation given by the following PDE:
\begin{align*}
    &u\cdot \frac{du}{dx} = \nu \cdot \frac{d^2 u}{d x^2}, \,\,\, u(a) = \alpha, u(b) = \beta, 
\end{align*}
where $u(x)$ is defined over the interval $x\in [a,b]$, $\nu > 0$ is the viscosity parameter, and $\alpha$ and $\beta$ are boundary parameters. We consider the point at which the solution changes its sign as the quantity of interest. It is experimentally known that this QoI is particularly sensitive to the choice of $\alpha$ and $\beta$ and not to the viscosity. Therefore, we fix $\nu = 0.1$, $[a,b] = [-1,1]$, and vary $\alpha,\beta \in [0.8,1.2] \times [-1.2,-0.8]$. We subtract the QoI obtained for these parameters by the minimum value to ensure that the function is always positive and again fit with a single neuron model of the form $\relu(p(\alpha,\beta))$, where $p$ has total degree $7$. Results in Figure \ref{fig:quant} align with the previous test problems: leverage score sampling shows a clear improvement over uniform sampling. For this problem, the improvement was less significant for a larger number of samples, suggesting that the model was simple enough that both the uniform and leverage score methods were able to eventually obtain a near-optimal fit.

\section{Discussion and Future Work}
 \label{sec:limitations}
 We believe our main theoretical result can be improved in a number of ways. 
 Most importantly, an ideal result would obtain a near-linear dependence on $d$ instead of a dependence on $d^2$, mirroring the $O(d\log d)$ sample complexity obtained by leverage score sampling for the active linear regression problem. The $d^2$ dependence is an inherent artifact of our $\eps$-net analysis; possible approaches to improve this include appealing to a more careful net construction, as in \cite{MuscoMuscoWoodruff:2022}, or more directly reducing to matrix concentration, as was done in recent work to obtain a near-linear dependence for a related problem involving $\ell_1$ embeddings of vectors transformed by Lipschitz non-linearities \citep{MaiRaoMusco:2021}.  
 
 \begin{figure}[hbt!]
 \centering
 	\includegraphics[scale=.3]{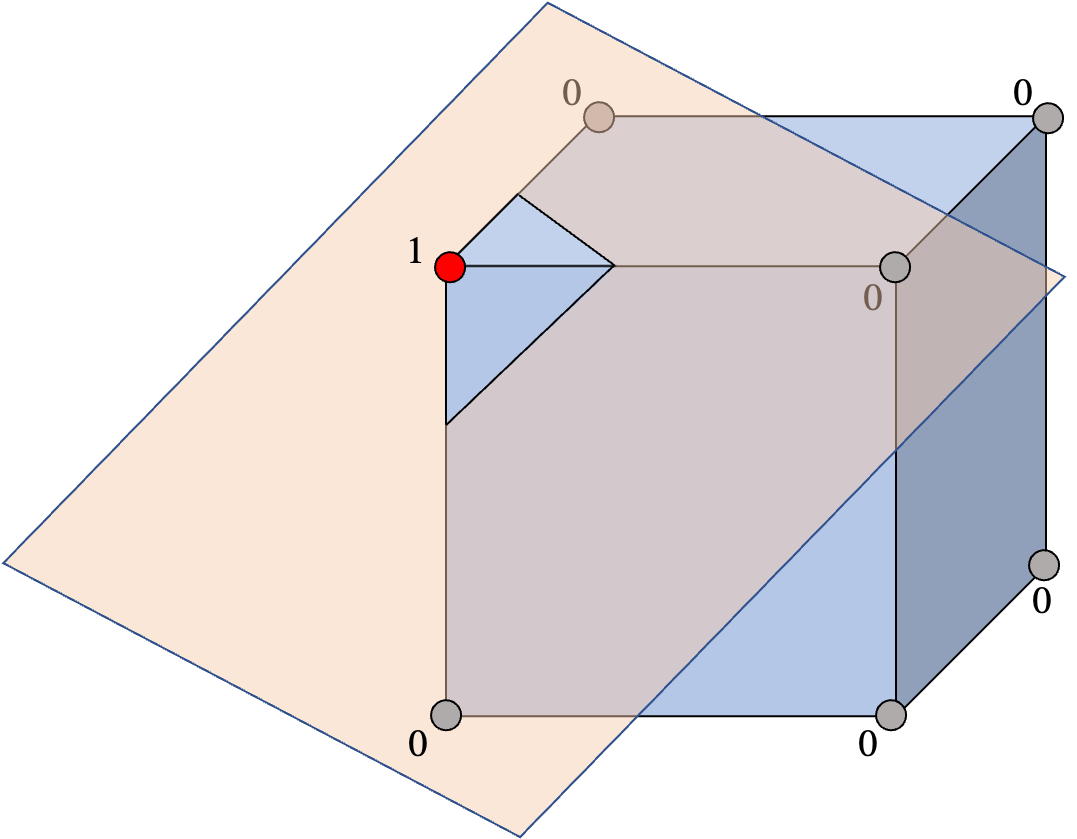}
 	\caption{Hard instance for obtaining relative error.}
 	\label{fig:lower_bound}
 \end{figure}
 
One might also hope to improve the error bound of Theorem \ref{thm:main_result}. For example, it would be ideal to obtain a pure relative error bound of the form $\|f(\X\hat{\w}) - \y\|_2^2 \leq C \cdot OPT$. Unfortunately, we can argue that this is not possible without taking a number of samples exponential in $d$. Consider a $(d+1)$ data matrix $\X$ whose first $d$  columns contain all of the $2^d$ vertices of the $d$ dimensional hypercube (i.e., there is a row containing every binary vector of length $d$). Let the last column of $\X$ be the all-ones vector. Consider the $1$-Lipschitz non-linearity $f(a) = \relu(a)$ and let $\y = f(\X\w)$ for some ground truth $\w$, in which case a pure relative error guarantee requires exactly recovering $\y$ (since $OPT = 0$). As visualized in Figure \ref{fig:lower_bound}, since there is a hyperplane separating any vertex in the hypercube from all other vertices, for any $i$ it is possible to find a $\w_i$ such that $\relu(\X\w_i)$ evaluates to $1$ in its $i^\text{th}$ coordinate, and $0$ everywhere else. Without observing at least $\Omega(2^d)$ entries from $\y$, we cannot distinguish between the case when $\y = \relu(\X\w_i)$ for a randomly chosen $i$ or $\y = \relu(\X\mathbf{0})$.

Finally, we note that a major open direction for future research is to obtain provable active learning methods in the agnostic setting for the more challenging multi-index model, or in the case when $f$ is not known in advance (and must be learned as part of the training process). We have some preliminary progress for the case where $f$ is unknown, but defer a full discussion to future work.

\bibliographystyle{plainnat}
\bibliography{ref} 

\begin{thebibliography}{50}
\providecommand{\natexlab}[1]{#1}
\providecommand{\url}[1]{\texttt{#1}}
\expandafter\ifx\csname urlstyle\endcsname\relax
  \providecommand{\doi}[1]{doi: #1}\else
  \providecommand{\doi}{doi: \begingroup \urlstyle{rm}\Url}\fi

\bibitem[Adcock et~al.(2022{\natexlab{a}})Adcock, Brugiapaglia, Dexter, and
  Morage]{AdcockBrugiapagliaDexter:2022}
Ben Adcock, Simone Brugiapaglia, Nick Dexter, and Sebastian Morage.
\newblock Deep neural networks are effective at learning high-dimensional
  hilbert-valued functions from limited data.
\newblock In \emph{Proceedings of the 2nd Mathematical and Scientific Machine
  Learning Conference}, volume 145 of \emph{Proceedings of Machine Learning
  Research}, pages 1--36, 2022{\natexlab{a}}.

\bibitem[Adcock et~al.(2022{\natexlab{b}})Adcock, Cardenas, Dexter, and
  Moraga]{AdcockCardenasDexter:2022}
Ben Adcock, Juan~M. Cardenas, Nick Dexter, and Sebastian Moraga.
\newblock \emph{Towards Optimal Sampling for Learning Sparse Approximations in
  High Dimensions}, pages 9--77.
\newblock Springer International Publishing, 2022{\natexlab{b}}.

\bibitem[Avron et~al.(2019)Avron, Kapralov, Musco, Musco, Velingker, and
  Zandieh]{AvronKapralovMusco:2019}
Haim Avron, Michael Kapralov, Cameron Musco, Christopher Musco, Ameya
  Velingker, and Amir Zandieh.
\newblock A universal sampling method for reconstructing signals with simple
  fourier transforms.
\newblock In \emph{\STOC{2019}}, 2019.

\bibitem[Binev et~al.(2017)Binev, Cohen, Dahmen, DeVore, Petrova, and
  Wojtaszczyk]{BinevCohenDahmen:2017}
Peter Binev, Albert Cohen, Wolfgang Dahmen, Ronald DeVore, Guergana Petrova,
  and Przemyslaw Wojtaszczyk.
\newblock Data assimilation in reduced modeling.
\newblock \emph{SIAM/ASA Journal on Uncertainty Quantification}, 5\penalty0
  (1):\penalty0 1--29, 2017.

\bibitem[Cand{\`{e}}s(2003)]{Candes:2003}
Emmanuel~J. Cand{\`{e}}s.
\newblock Ridgelets: estimating with ridge functions.
\newblock \emph{The Annals of Statistics}, 31\penalty0 (5):\penalty0
  1561--1599, 2003.

\bibitem[Chen and Derezinski(2021)]{ChenDerezinski:2021}
Xue Chen and Michal Derezinski.
\newblock Query complexity of least absolute deviation regression via robust
  uniform convergence.
\newblock In \emph{\COLT{2021}}, volume 134, pages 1144--1179, 2021.

\bibitem[Chen and Price(2019)]{ChenPrice:2019a}
Xue Chen and Eric Price.
\newblock Active regression via linear-sample sparsification active regression
  via linear-sample sparsification.
\newblock In \emph{\COLT{2019}}, 2019.

\bibitem[Chkifa et~al.(2018)Chkifa, Dexter, Tran, and
  Webster]{ChkifaDexterTran:2018}
Abdellah Chkifa, Nick Dexter, Hoang Tran, and Clayton~G. Webster.
\newblock Polynomial approximation via compressed sensing of high-dimensional
  functions on lower sets.
\newblock \emph{Math. Comp.}, 87\penalty0 (311):\penalty0 1415--1450, 2018.

\bibitem[Cohen and DeVore(2015)]{CohenDeVore:2015}
Albert Cohen and Ronald DeVore.
\newblock Approximation of high-dimensional parametric {PDE}s.
\newblock \emph{Acta Numerica}, 24:\penalty0 1, 2015.

\bibitem[Cohen and Migliorati(2017)]{CohenMigliorati:2017}
Albert Cohen and Giovanni Migliorati.
\newblock Optimal weighted least-squares methods.
\newblock \emph{SMAI Journal of Computational Mathematics}, 3:\penalty0
  181--203, 2017.

\bibitem[Cohen et~al.(2012)Cohen, Daubechies, DeVore, Kerkyacharian, and
  Picard]{CohenDaubechiesDeVore:2012}
Albert Cohen, Ingrid Daubechies, Ronald DeVore, Gerard Kerkyacharian, and
  Dominique Picard.
\newblock Capturing ridge functions in high dimensions from point queries.
\newblock \emph{Constructive Approximation}, 35\penalty0 (2):\penalty0
  225--243, 2012.

\bibitem[Cohen et~al.(2015)Cohen, Lee, Musco, Musco, Peng, and
  Sidford]{CohenLeeMusco:2015}
Michael~B. Cohen, Yin~Tat Lee, Cameron Musco, Christopher Musco, Richard Peng,
  and Aaron Sidford.
\newblock Uniform sampling for matrix approximation.
\newblock In \emph{\ITCS{2015}}, pages 181--190, 2015.

\bibitem[Cohen et~al.(2017)Cohen, Musco, and Musco]{CohenMuscoMusco:2017}
Michael~B. Cohen, Cameron Musco, and Christopher Musco.
\newblock Input sparsity time low-rank approximation via ridge leverage score
  sampling.
\newblock In \emph{\SODA{2017}}, pages 1758--1777, 2017.

\bibitem[Constantine et~al.(2016)Constantine, del Rosario, and
  Iaccarino]{ConstantineRosarioIaccarino:2016}
Paul~G. Constantine, Zachary del Rosario, and Gianluca Iaccarino.
\newblock Many physical laws are ridge functions.
\newblock \emph{\arXiv{1605.07974}}, 2016.

\bibitem[Constantine et~al.(2017)Constantine, Eftekhari, Hokanson, and
  Ward]{ConstantineEftekhariHokanson:2017}
Paul~G. Constantine, Armin Eftekhari, Jeffrey Hokanson, and Rachel~A. Ward.
\newblock A near-stationary subspace for ridge approximation.
\newblock \emph{Computer Methods in Applied Mechanics and Engineering},
  326:\penalty0 402--421, 2017.

\bibitem[Dasgupta et~al.(2008)Dasgupta, Drineas, Harb, Kumar, and
  Mahoney]{DasguptaDrineasHarb:2008}
Anirban Dasgupta, Petros Drineas, Boulos Harb, Ravi Kumar, and Michael~W.
  Mahoney.
\newblock Sampling algorithms and coresets for lp regression.
\newblock In \emph{\SODA{2008}}, pages 932--941, 2008.

\bibitem[Diakonikolas et~al.(2020{\natexlab{a}})Diakonikolas, Goel, Karmalkar,
  Klivans, and Soltanolkotabi]{DiakonikolasGoelKarmalkar:2020}
Ilias Diakonikolas, Surbhi Goel, Sushrut Karmalkar, Adam~R. Klivans, and Mahdi
  Soltanolkotabi.
\newblock Approximation schemes for relu regression.
\newblock In \emph{\COLT{2020}}, volume 125, pages 1452--1485,
  2020{\natexlab{a}}.

\bibitem[Diakonikolas et~al.(2020{\natexlab{b}})Diakonikolas, Kane, and
  Zarifis]{DiakonikolasKaneZarifis:2020}
Ilias Diakonikolas, Daniel Kane, and Nikos Zarifis.
\newblock Near-optimal sq lower bounds for agnostically learning halfspaces and
  relus under gaussian marginals.
\newblock In \emph{\NIPS{2020}}, pages 13586--13596, 2020{\natexlab{b}}.

\bibitem[Diakonikolas et~al.(2022{\natexlab{a}})Diakonikolas, Kane, Manurangsi,
  and Ren]{DiakonikolasKaneManurangsi:2022}
Ilias Diakonikolas, Daniel Kane, Pasin Manurangsi, and Lisheng Ren.
\newblock Hardness of learning a single neuron with adversarial label noise.
\newblock In \emph{\AISTATS{2022}}, volume 151, pages 8199--8213,
  2022{\natexlab{a}}.

\bibitem[Diakonikolas et~al.(2022{\natexlab{b}})Diakonikolas, Kontonis, Tzamos,
  and Zarifis]{DiakonikolasKontonisTzamos:2022}
Ilias Diakonikolas, Vasilis Kontonis, Christos Tzamos, and Nikos Zarifis.
\newblock Learning a single neuron with adversarial label noise via gradient
  descent.
\newblock In \emph{\COLT{2022}}, volume 178, pages 4313--4361,
  2022{\natexlab{b}}.

\bibitem[Drineas et~al.(2006)Drineas, Mahoney, and
  Muthukrishnan]{DrineasMahoneyMuthukrishnan:2006}
Petros Drineas, Michael~W. Mahoney, and S.~Muthukrishnan.
\newblock Sampling algorithms for {$\ell_2$} regression and applications.
\newblock In \emph{\SODA{2006}}, pages 1127--1136, 2006.

\bibitem[Drineas et~al.(2008)Drineas, Mahoney, and
  Muthukrishnan]{DrineasMahoneyMuthukrishnan:2008}
Petros Drineas, Michael~W. Mahoney, and S.~Muthukrishnan.
\newblock Relative-error {CUR} matrix decompositions.
\newblock \emph{SIAM J. Matrix Anal. Appl.}, 30\penalty0 (2):\penalty0
  844--881, 2008.

\bibitem[Erd\'{e}lyi et~al.(2020)Erd\'{e}lyi, Musco, and
  Musco]{ErdelyiMuscoMusco:2020}
Tam\'{a}s Erd\'{e}lyi, Cameron Musco, and Christopher Musco.
\newblock Fourier sparse leverage scores and approximate kernel learning.
\newblock \emph{\NIPS{2020}}, 2020.

\bibitem[Feldman and Langberg(2011)]{FeldmanLangberg:2011}
Dan Feldman and Michael Langberg.
\newblock A unified framework for approximating and clustering data.
\newblock In \emph{\STOC{2011}}, pages 569--578, 2011.

\bibitem[Fornasier et~al.(2012)Fornasier, Schnass, and
  Vybiral]{FornasierSchnassVybiral:2012}
Massimo Fornasier, Karin Schnass, and Jan Vybiral.
\newblock Learning functions of few arbitrary linear parameters in high
  dimensions.
\newblock \emph{Foundations of Computational Mathematics}, 12\penalty0
  (2):\penalty0 229--262, 2012.

\bibitem[Goel et~al.(2017)Goel, Kanade, Klivans, and
  Thaler]{GoelKanadeKlivans:2017}
Surbhi Goel, Varun Kanade, Adam Klivans, and Justin Thaler.
\newblock Reliably learning the relu in polynomial time.
\newblock In \emph{\COLT{2017}}, volume~65, pages 1004--1042, 2017.

\bibitem[Goel et~al.(2019)Goel, Karmalkar, and
  Klivans]{GoelKarmalkarKlivans:2019}
Surbhi Goel, Sushrut Karmalkar, and Adam Klivans.
\newblock Time/accuracy tradeoffs for learning a relu with respect to gaussian
  marginals.
\newblock In \emph{\NIPS{2019}}, 2019.

\bibitem[Hampton and Doostan(2015{\natexlab{a}})]{HamptonDoostan:2015}
Jerrad Hampton and Alireza Doostan.
\newblock Coherence motivated sampling and convergence analysis of least
  squares polynomial chaos regression.
\newblock \emph{Comput. Method. Appl. M.}, 290:\penalty0 73--97,
  2015{\natexlab{a}}.

\bibitem[Hampton and Doostan(2015{\natexlab{b}})]{HamptonDoostan:2015b}
Jerrad Hampton and Alireza Doostan.
\newblock Compressive sampling of polynomial chaos expansions: Convergence
  analysis and sampling strategies.
\newblock \emph{Journal of Computational Physics}, 280:\penalty0 363--386,
  2015{\natexlab{b}}.

\bibitem[Hokanson and Constantine(2018)]{HokansonConstantine:2018}
Jeffrey~M. Hokanson and Paul~G. Constantine.
\newblock Data-driven polynomial ridge approximation using variable projection.
\newblock \emph{SIAM Journal on Scientific Computing}, 40\penalty0 (3), 2018.

\bibitem[Klusowski and Barron(2018)]{KlusowskiBarron:2018}
Jason~M. Klusowski and Andrew~R. Barron.
\newblock Approximation by combinations of relu and squared relu ridge
  functions with $\ell^1$ and $\ell^0$ controls.
\newblock \emph{IEEE Transactions on Information Theory}, 64\penalty0
  (12):\penalty0 7649--7656, 2018.

\bibitem[Lassila and Rozza(2010)]{LassilaRozza:2010}
Toni Lassila and Gianluigi Rozza.
\newblock Parametric free-form shape design with {PDE} models and reduced basis
  method.
\newblock \emph{Computer Methods in Applied Mechanics and Engineering},
  199\penalty0 (23):\penalty0 1583--1592, 2010.

\bibitem[Le~Ma{\^\i}tre and Knio(2010)]{Le-MaitreKnio:2010}
Olivier~P. Le~Ma{\^\i}tre and Omar~M. Knio.
\newblock \emph{Spectral methods for uncertainty quantification : with
  applications to computational fluid dynamics}.
\newblock Scientific computation. Springer Netherlands, Dordrecht, New York,
  2010.

\bibitem[Mai et~al.(2021)Mai, Rao, and Musco]{MaiRaoMusco:2021}
Tung Mai, Anup~B. Rao, and Cameron Musco.
\newblock Coresets for classification -- simplified and strengthened.
\newblock In \emph{\NIPS{2021}}, 2021.

\bibitem[Meyer et~al.(2023)Meyer, Musco, Musco, and David
  P.~Woodruff]{MeyerMuscoMusco:2023}
Raphael Meyer, Cameron Musco, Christopher Musco, and Samson~Zhou David
  P.~Woodruff.
\newblock Near-linear sample complexity for lp polynomial regression.
\newblock In \emph{\SODA{2023}}, 2023.

\bibitem[Munteanu et~al.(2018)Munteanu, Schwiegelshohn, Sohler, and
  Woodruff]{MunteanuSchwiegelshohnSohler:2018}
Alexander Munteanu, Chris Schwiegelshohn, Christian Sohler, and David Woodruff.
\newblock On coresets for logistic regression.
\newblock In \emph{\NIPS{2018}}, volume~31, 2018.

\bibitem[Musco and Musco(2017)]{MuscoMusco:2017}
Cameron Musco and Christopher Musco.
\newblock Recursive sampling for the {N}ystr{\"o}m method.
\newblock In \emph{\NIPS{2017}}, pages 3833--3845, 2017.

\bibitem[Musco et~al.(2022)Musco, Musco, Woodruff, and
  Yasuda]{MuscoMuscoWoodruff:2022}
Cameron Musco, Christopher Musco, David~P. Woodruff, and Taisuke Yasuda.
\newblock Active linear regression for $\ell_p$ norms and beyond.
\newblock In \emph{\FOCS{2022}}, 2022.

\bibitem[O'Leary-Roseberry et~al.(2022)O'Leary-Roseberry, Villa, Chen, and
  Ghattas]{OLeary-RoseberryVillaChen:2022}
Thomas O'Leary-Roseberry, Umberto Villa, Peng Chen, and Omar Ghattas.
\newblock Derivative-informed projected neural networks for high-dimensional
  parametric maps governed by pdes.
\newblock \emph{Computer Methods in Applied Mechanics and Engineering}, 388,
  2022.

\bibitem[Pinkus(1997)]{Pinkus:1997}
Allan Pinkus.
\newblock Approximating by ridge functions.
\newblock \emph{Surface fitting and multiresolution methods}, pages 279--292,
  1997.

\bibitem[Pinkus(2015)]{Pinkus:2015}
Allan Pinkus.
\newblock \emph{Ridge functions}, volume 205.
\newblock Cambridge University Press, 2015.

\bibitem[Pukelsheim(2006)]{Pukelsheim:2006}
Friedrich Pukelsheim.
\newblock \emph{Optimal Design of Experiments}.
\newblock Society for Industrial and Applied Mathematics, 2006.

\bibitem[Rao et~al.(2017)Rao, Ganti, Balzano, Willett, and
  Nowak]{RaoGantiBalzano:2017}
Nikhil Rao, Ravi Ganti, Laura Balzano, Rebecca Willett, and Robert Nowak.
\newblock On learning high-dimensional structured single index models.
\newblock In \emph{\AAAI{2017}}, 2017.

\bibitem[Rauhut and Ward(2012)]{RauhutWard:2012}
Holger Rauhut and Rachel Ward.
\newblock Sparse {L}egendre expansions via $\ell 1$-minimization.
\newblock \emph{Journal of Approximation Theory}, 164\penalty0 (5):\penalty0
  517 -- 533, 2012.

\bibitem[Sarlos(2006)]{Sarlos:2006}
Tamas Sarlos.
\newblock Improved approximation algorithms for large matrices via random
  projections.
\newblock In \emph{\FOCS{2006}}, pages 143--152, 2006.

\bibitem[Spielman and Srivastava(2011)]{SpielmanSrivastava:2011}
Daniel~A. Spielman and Nikhil Srivastava.
\newblock Graph sparsification by effective resistances.
\newblock \emph{SIAM Journal on Computing}, 40\penalty0 (6):\penalty0
  1913--1926, 2011.
\newblock \pSTOC{2008}.

\bibitem[Tyagi and Cevher(2012)]{TyagiCevher:2012}
Hemant Tyagi and Volkan Cevher.
\newblock Active learning of multi-index function models.
\newblock In \emph{\NIPS{2012}}, pages 1466--1474, 2012.

\bibitem[Vershynin(2012)]{Vershynin:2012}
Roman Vershynin.
\newblock \emph{Introduction to the non-asymptotic analysis of random
  matrices}.
\newblock Cambridge University Press, 2012.

\bibitem[Woodruff(2014)]{Woodruff:2014}
David~P. Woodruff.
\newblock Sketching as a tool for numerical linear algebra.
\newblock \emph{Foundations and Trends in Theoretical Computer Science},
  10\penalty0 (1--2):\penalty0 1--157, 2014.

\bibitem[Yehudai and Shamir(2020)]{YehudaiOhad:2020}
Gilad Yehudai and Ohad Shamir.
\newblock Learning a single neuron with gradient methods.
\newblock In \emph{\COLT{2020}}, volume 125, pages 3756--3786, 2020.

\end{thebibliography}
\clearpage

\section{Appendix}
\label{sec:app}
\begin{proof}[Proof of Lemma \ref{prop:bern}]
    Let $\x_i$ denote the $i^\text{th}$ row of $\X$ and let $ \u = f(\X\w_1) - f(\X \w_2)$. Our goal is to show that $\|\S\u\|_2^2$ approximately equals $\|\u\|_2^2$ with high probability. Let $j_i \in [n]$ be the index of the row from $\X$ selected by the $i^\text{th}$ row in $\S$. We have that  $\norm{\S \u}_2^2 = \frac{1}{m}\sum_{i=1}^m \frac{u_{j_i}^2}{p_{j_i}}$, where $p_{j_i} = \tau_{j_i}(\X)/\rank(\X)$. So we first observe that $\E \norm{\S \u}_2^2 = \norm{\u}_2^2$. Next, we will show that the random variable $\|\S\u\|_2^2$ concentrates around it's expectation by applying Berstein's inequality. To do so, we need to bound the variance of each term in the sum, $\frac{1}{m}\sum_{i=1}^m \frac{u_{j_i}^2}{p_{j_i}}$. We defining $\v = \X\w_1 - \X\w_2$ and observing that, since $f$ is $L$-Lipschitz, for every $i \in [n]$,  
	\begin{align}\label{eqn:bernlip}
		&u_i = \abs{f(\langle \x_i, \w_1\rangle) - f(\langle \x_i, \w_2\rangle)}_i\nonumber\\
  &\le L \cdot \abs{\langle \x_i, \w_1\rangle -\langle \x_i, \w_2\rangle}_i \nonumber\\
  &\leq L v_i. 
	\end{align}
	  We then have that:
	\begin{align*}
	\Var\left[\frac{u_{j_i}^2}{ p_{j_i}}\right] \leq \E\left[\left(\frac{u_{j_i}^2}{p_{j_i}}\right)^2\right] = \sum_{k=1}^n \frac{u_k^4}{p_k^2}\cdot p_k \leq \sum_{k=1}^n \frac{L^4v_k^4 \rank(\X)}{\tau_k(\X)}. 
	\end{align*}
	In the last step we have used the upper bound from \eqref{eqn:bernlip}, and the fact that $p_{k} = \tau_{k}(\X)/\rank(\X)$. From the definition of leverage scores (Definition \ref{def:lev_scores}), and the fact that $\v$ lies in the variance as follows:
	\begin{align*}
		\Var\left[\frac{u_{j_i}^2}{ p_{j_i}}\right] &\leq L^4\cdot \sum_{k=1}^n v_k^2 \|\v\|_2^2 \rank(\X)\\
            &= L^4\cdot \|\v\|_2^4 \cdot \rank(\X)\\ 
            &\leq  L^4\cdot d\|\v\|_2^4.
	\end{align*}
	Moreover, using the sames bounds as above, we always have that $\frac{u_{j_i}^2}{p_{j_i}} \leq \max_k L^2 v_k^2 \cdot \frac{\rank(\X)}{\tau_k(\X)} \leq L^2 \cdot d\|\v\|_2^2$. So, we can apply Bernstein's inequality to conclude that:
	\begin{align*}
		\Pr\left[\left|\|\S\u\|_2^2 - \|\u\|_2^2 \right| \geq t \right]
  &\leq 2\exp\left(-\frac{m t^2/2}{ L^4\cdot d\|\v\|_2^4 + t\cdot L^2 \cdot d\|\v\|_2^2/3}\right). 
	\end{align*}
	Setting $m = \frac{3d\log(2/\delta)}{\eps^2}$ and $t = \eps L^2 \|\v\|_2^2$ and plugging in we conclude that: 
	\begin{align*}
		\Pr\left[\left|\|\S\u\|_2^2 - \|\u\|_2^2 \right| \geq \eps L^2 \|\v\|_2^2 \right] 
  &\leq 2\exp\left(-\frac{\frac{3}{2} d\log(\delta/2) L^4 \|\v\|_2^4}{(1+\epsilon/3)d L^4 \|\v\|_2^4}\right)\leq \delta. 
	\end{align*}
	This completes the bound.
\end{proof}
\begin{proof}[Proof of Lemma \ref{lemm:net_construction}]
    Let $N$ be an $(\eps R)$-net in the Euclidean norm on $\mathcal{B}^d(R)$. I.e. for every $\v\in \mathcal{B}^d(R)$, there should be some point $\z\in N$ such that $\|\z - \v\|_2 \leq \eps R$. It is well known that such an $N$ exists with cardinality $|N| \leq \left(1+\frac{2}{\eps})\right)^d$ (see e.g. Lemma 5.2 in \cite{Vershynin:2012}). Applying Lemma \ref{prop:bern} with $\delta = \frac{1}{50|N|}$ and error parameter $\epsilon^2$ and combining with a union bound, we conclude that as long as $m \geq  c \frac{d^2\log(1/\eps)}{\eps^4}$ for a fixed constant $c$, then with probability $99/100$, for all $\z\in N$,
	\begin{align}
		\label{eq:bound_for_net}
		\norm{f(\X\z) - f(\X\w^*)}_2^2 \in \left[\norm{\S f(\X\z) - \S f(\X\w^*)}_2^2 \pm \eps^2L^2 \|\X\z - \X \w^*\|_2^2\right].
	\end{align}
 Furthermore, when  $m \geq  c \frac{d^2\log(1/\eps)}{\eps^4}$, by the subspace embedding from \ref{lem:subspace}, we have that for all $\mathbf{a}, \mathbf{b}$, 
 \begin{align}
 \label{eq:subspace_special}
 \|\S\X\mathbf{a} - \S\X\mathbf{b}\|_2^2 \leq 2 \|\X\mathbf{a} - \X\mathbf{b}\|_2^2 
\end{align} 
with probability $99/100$.
 
	Now, let $\z^*$ be the closest point to $\wh$ in $N$. I.e., $\z^* = \argmin_{z \in N}\norm{\z - \wh}_2$. Applying \eqref{eq:bound_for_net}, \eqref{eq:subspace_special}, and the fact that for any two vectors $\mathbf{a},\mathbf{b}$, $\|\mathbf{a} + \mathbf{b}\|_2^2 \leq 2 \|\mathbf{a}\|_2 + 2\|\mathbf{b}\|_2^2$, we have the following inequalities.
	\begin{align*}\label{eq:introduce_netpoint}
		\|f(\X\wh) &- f(\X\w^*)\|_2^2 \le 2\norm{ f(\X\z^*) - f(\X\w^*)}_2^2 + 2\norm{ f(\X \wh) - f(\X\z^*)}_2^2 \\
		&\leq2\norm{\S f(\X\z^*) - \S f(\X\w^*)}_2^2 + 2 \eps^2L^2 \|\X\z^* - \X \w^*\|_2^2+  2\norm{ f(\X \wh) - f(\X\z^*)}_2^2\\
		&\leq  4\norm{\S f(\X\wh) - \S f(\X\w^*)}_2^2 + 4\norm{\S f(\X\z^*) - \S f(\X\wh)}_2^2 +2 \eps^2L^2 \|\X\z^* - \X \w^*\|_2^2 +  2\norm{ f(\X \wh) - f(\X\z^*)}_2^2\\
		&\leq\ 4  \norm{\S f(\X\wh) - \S f(\X\w^*)}_2^2 + 4L^2\norm{\S \X\z^* - \S \X\wh}_2^2 + 2\eps^2L^2 \|\X\z^* - \X \w^*\|_2^2 +2 L^2 \|\X\wh - \X \z^*\|_2^2\\
		&\leq\ 4  \norm{\S f(\X\wh) - \S f(\X\w^*)}_2^2 + 4L^2\cdot 2\cdot \norm{\X\z^* - \X\wh}_2^2 
  + 4\eps^2L^2 (R^2 + \|\X \w^*\|_2^2) + 2 L^2 \|\X\wh - \X \z^*\|_2^2\\
		&\leq  4  \norm{\S f(\X\wh) - \S f(\X\w^*)}_2^2 + 8\eps^2L^2R^2 + 4\eps^2L^2 R^2 + 4\|\X \w^*\|_2^2 +2 \eps^2 L^2R^2.
	\end{align*}
	Combining terms and adjusting constants on $\eps$ yields the bound.
\end{proof}
\end{document}